\documentclass[lettersize,journal]{IEEEtran}
\usepackage{amsmath,amsfonts}
\usepackage{algorithmic}
\usepackage{algorithm}

\usepackage{array}
\usepackage{amsthm}
\theoremstyle{remark} 
\usepackage[caption=false,font=footnotesize]{subfig} 

\usepackage{textcomp}
\usepackage{stfloats}
\usepackage{graphicx}  
\usepackage{booktabs}
\usepackage{url}
\usepackage{verbatim}
\usepackage{graphicx}
\usepackage{cite}
\newtheorem{theorem}{Theorem}
\newtheorem{lemma}{Lemma}
\newtheorem{assumption}{Assumption}
\begin{document}

\title{Confidence-Guided Human-AI Collaboration: Reinforcement Learning with Distributional Proxy Value Propagation for Autonomous Driving}

\author{Zeqiao Li,~Yijing Wang,~Haoyu Wang,~Zheng Li,~Peng Li,~Zhiqiang Zuo,~\IEEEmembership{Senior Member,~IEEE},~Chuan Hu
	\thanks{This work was supported in part by the National Natural Science Foundation of China under Grant 62173243, Grant 61933014, Grant 62403348, and the Young Scientists Fund of the National Natural Science Foundation of Tianjin, China, under Grant 23JCQNJC01780, and the Postdoctoral Fellowship Program of CPSF under Grant GZC20241208, Grant 2024M762357, and the Foundation of Key Laboratory of System Control and Information Processing, Ministry of Education, P.R. China, under Grant Scip20240116.\\
		\indent The authors are with the Tianjin Key Laboratory of Intelligent Unmanned Swarm Technology and System, School of Electrical and Information Engineering, Tianjin University, Tianjin 300072, China; Haoyu Wang is also with Key Laboratory of System Control and Information Processing, Ministry of Education of China, Shanghai, 200240 (email:  lizeqiao@tju.edu.cn; yjwang@tju.edu.cn; why2014@tju.edu.cn; zhengl@tju.edu.cn; lipeng\_2017@tju.edu.cn; zqzuo@tju.edu.cn). Chuan~Hu is with the School of Mechanical Engineering, Shanghai Jiao Tong University, Shanghai 200240, China (e-mail: chuan.hu@sjtu.edu.cn).

	}
}

\markboth{Journal of \LaTeX\ Class Files,~Vol.~14, No.~8, August~2021}%
{Shell \MakeLowercase{\textit{et al.}}: A Sample Article Using IEEEtran.cls for IEEE Journals}


\maketitle

\begin{abstract} 
    Autonomous driving promises significant advancements in mobility, road safety and traffic efficiency, yet reinforcement learning and imitation learning face safe-exploration and distribution-shift challenges. Although human-AI collaboration alleviates these issues, it often relies heavily on extensive human intervention, which increases costs and reduces efficiency. This paper develops a confidence-guided human-AI collaboration (C-HAC) strategy to overcome these limitations. First, C-HAC employs a distributional proxy value propagation method within the distributional soft actor-critic (DSAC) framework. By leveraging return distributions to represent human intentions C-HAC achieves rapid and stable learning of human-guided policies with minimal human interaction. Subsequently, a shared control mechanism is activated to integrate the learned human-guided policy with a self-learning policy that maximizes cumulative rewards. This enables the agent to explore independently and continuously enhance its performance beyond human guidance. Finally, a policy confidence evaluation algorithm capitalizes on DSAC's return distribution networks to facilitate dynamic switching between human-guided and self-learning policies via a confidence-based intervention function. This ensures the agent can pursue optimal policies while maintaining safety and performance guarantees. Extensive experiments across diverse driving scenarios reveal that C-HAC significantly outperforms conventional methods in terms of safety, efficiency, and overall performance, achieving state-of-the-art results. The effectiveness of the proposed method is further validated through real-world road tests in complex traffic conditions. The videos and code are available at: https://github.com/lzqw/C-HAC.
\end{abstract}

\begin{IEEEkeywords}
  Autonomous driving, Human-AI collaboration, Deep reinforcement learning.
\end{IEEEkeywords}

\section{Introduction}
\IEEEPARstart{A}{utonomous} driving technology is revolutionizing transportation by offering enhanced mobility, improved safety, increased traffic efficiency, and reduced environmental impacts. These advancements have garnered significant attention from both academia and industry\cite{Learn1Day,E2ESelf,DenseRL,SurveyRLIL01}. Recent advances in artificial intelligence (AI) have propelled learning-based policies to the forefront of autonomous driving research, providing an alternative to traditional modular approaches that segment the driving system into distinct components such as perception, localization, planning, and control\cite{SurveyRLIL02,survey4}. Learning-based policies, particularly end-to-end approaches, have demonstrated potential in enhancing the perception and decision-making capabilities of AI systems in autonomous vehicles (AVs)\cite{survey5}. To further enhance learning-based approaches in autonomous driving, reinforcement learning (RL), imitation learning (IL) and human-AI collaboration (HAC) learning have been extensively explored.\\
\indent RL is a powerful paradigm for training autonomous systems through trial-and-error interactions with the environment. It enables agents to establish causal relationships among observations, actions, and outcomes\cite{RLSurvey01}. A reward function guides the learning process and encodes desired behaviors, allowing extensive exploration of the environment to maximize cumulative rewards. This approach often yields robust policies capable of handling complex tasks \cite{RLSurvey02,RLSurvey03,Roach}. Nevertheless, several challenges restrict the applicability of RL in safety-critical scenarios. Designing a reward function that aligns with diverse human preferences is nontrivial, and flawed reward functions may lead to biased, misguided, or undesirable behaviors \cite{RLReward01,RLReward02,RLReward03}. Moreover, the exploratory nature of RL frequently induces unsafe situations during both training and testing. The low sample efficiency in agent-environment interactions further exacerbates computational costs and prolongs training \cite{RLSample01,RLSample02}. In addition, although policies trained exclusively with RL can be technically safe, they often lack the natural, human-like behaviors critical for coordinating with other vehicles in real-world driving scenarios \cite{RLSafe}. Therefore, there is a pressing need for approaches that mitigate these limitations—particularly unsafe exploration, reward design pitfalls, and the lack of human-like behavior—to fully realize RL's potential in autonomous driving.

IL is a promising approach for learning driving policies by replicating human driving behavior through demonstrated actions. Prominent methods include behavior cloning (BC)\cite{ILSurvey01} and inverse reinforcement learning (IRL)\cite{ILIRL01}. Leveraging expert demonstrations allows novice agents to avoid hazardous interactions during training, enhancing safety. Techniques such as BC and offline RL train agents on pre-existing datasets without environment interaction, further reducing risk \cite{IL01,IL02,IL03,IL04,IL05,IL06}, while IRL infers a reward function from demonstrations to guide desirable behaviors \cite{IRL01}. Despite these strengths, IL faces significant challenges. Distributional shift can cause compounding errors, pushing the agent away from its training distribution and resulting in control failures \cite{ILSHOURT02,ILSHOURT01}. Moreover, IL's success heavily depends on the quality of demonstrations, leaving it vulnerable to task variations \cite{ILSHOURT03}. Rare or atypical scenarios, often underrepresented in training data, can lead to erratic policy responses, especially in complex environments where expert data may be scarce or suboptimal \cite{ILSHOURT04}. Consequently, IL-based methods demand solutions that mitigate distributional shift, cope with suboptimal or limited demonstrations, and adapt to diverse driving conditions.\\
  \indent Human-AI collaboration (HAC) leverages human expertise to enhance safety and efficiency. Methods such as DAgger \cite{DAGGER} and its extensions \cite{DAGGEROther01,DAGGEROther02,DAGGEROther03} periodically request expert demonstrations to correct compounding errors in IL, while expert intervention learning (EIL) \cite{EIL} and intervention weighted regression (IWR) \cite{IWR} rely on human operators to intervene during exploration for safer state transitions. Other approaches integrate human evaluative feedback \cite{HIL01,HIL02,HIL03} or partial demonstrations with limited interventions, as in HACO \cite{HACO}, to guide the learning process while reducing human effort. Proxy value propagation (PVP) \cite{PVP} uses active human input to learn a proxy value function encoding human intentions, demonstrating robust performance across diverse tasks and action spaces. Despite these advances, many HAC methods assume near-optimal human guidance, which can be prohibitively expensive or suboptimal in practice \cite{HILSHORT01}. Humans may adopt conservative strategies in ambiguous scenarios, such as decelerating behind a slower vehicle for safety, even when a more effective approach (e.g., overtaking) might exist. Furthermore, continuous oversight places a heavy burden on human operators, hindering large-scale deployment \cite{ILIRL01}, while insufficient utilization of autonomous exploration data constrains overall learning efficiency \cite{HILSHORT03}. Balancing human involvement, agent safety, and learning efficiency thus remains a central challenge in HAC research.\\
\indent To reduce reliance on human guidance while ensuring efficient and safe policy learning, we propose a confidence-guided human-AI collaboration (C-HAC) strategy in this paper. It operates in two stages: a human-guided learning phase and a subsequent RL enhancement phase. In the former, Distributional proxy value propagation (D-PVP) encodes human intentions in the Distributional soft actor-critic (DSAC) algorithm. For the latter, the agent refines its policy independently. A shared control mechanism combines the learned human-guided policy with a self-learning policy, enabling exploration beyond human demonstrations. Additionally, a policy confidence evaluation algorithm leverages DSAC’s return distributions to switch between human-guided and self-learning policies, ensuring safety and reliable performance. Extensive experiments show that C-HAC quickly acquires robust driving skills with minimal human input and continues improving without ongoing human intervention, outperforming conventional RL, IL, and HAC methods in safety, efficiency, and overall results. The main contributions of this paper are as follows:
\begin{itemize} 
  \item Distributional proxy value propagation (D-PVP): A novel method integrating PVP into the DSAC framework is suggested. D-PVP enables the agent to learn effective driving policies with relatively few human-guided interactions, achieving competent performance within a short training duration.
  \item Shared control mechanism: A mechanism combining the learned human-guided policy with a self-learning policy is proposed. The self-learning policy is designed to maximize cumulative rewards, allowing the agent to explore independently and continuously improve its performance beyond human guidance. 
  \item Policy confidence evaluation algorithm: An algorithm leveraging DSAC's return distribution networks to facilitate dynamic switching between human-guided and self-learning policies via an intervention function is developed. This ensures the agent can pursue optimal policies while maintaining guaranteed safety and performance levels. 
\end{itemize}
\begin{figure*}[!t]
  \centering
  \includegraphics[width=7.in]{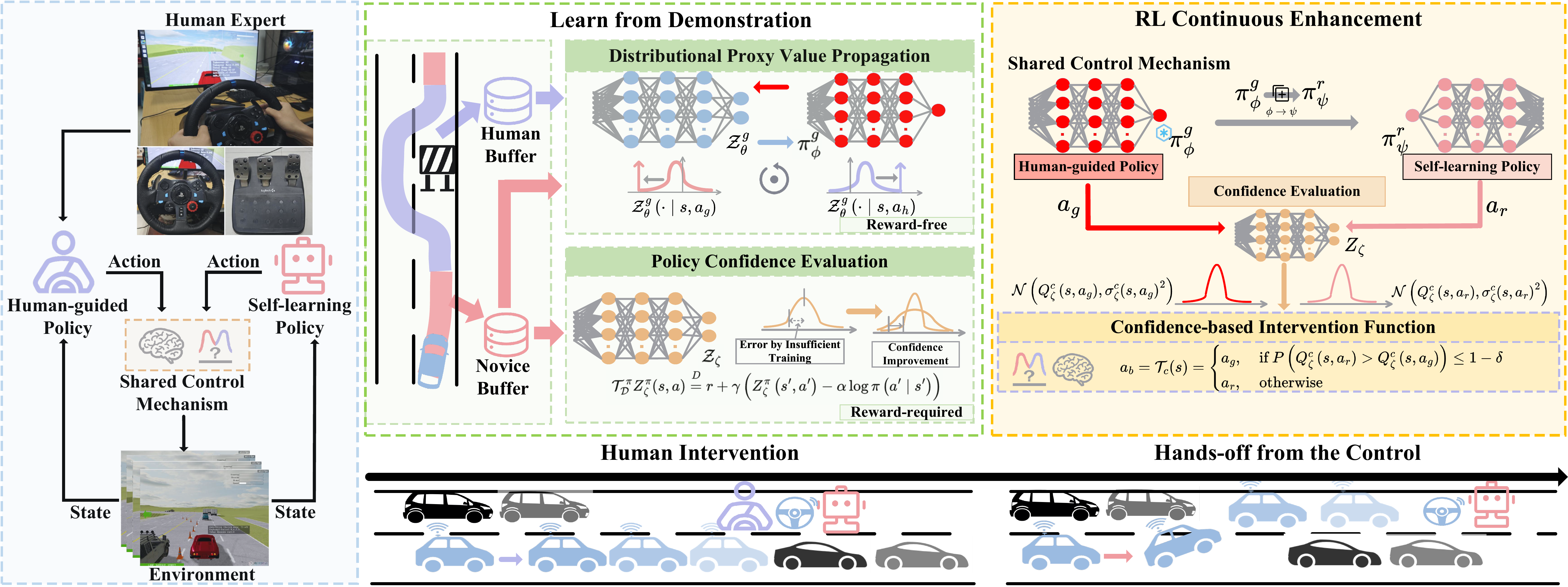}
  \caption{Overall framework of C-HAC}
  \label{fig_framework}
  \end{figure*}
\section{Methodology}
This section presents the proposed C-HAC framework, which integrates D-PVP, a shared control mechanism and policy confidence evaluation. As illustrated in Fig. \ref{fig_framework}, the framework consists of two phases. Initially, the agent employs D-PVP to learn from human demonstrations, allowing it to acquire safe and efficient driving policies through expert interventions. Subsequently, the agent continues to refine its policy via reinforcement learning, leveraging the shared control mechanism and policy confidence evaluation for safe and reliable exploration. To provide a comprehensive understanding, we will first offer a brief overview of conventional RL and HAC. We will then delve into detailed explanations of D-PVP, the shared control mechanism combined with policy confidence evaluation, and the entire training process.
\subsection{Human-AI Collaboration}
The policy learning of end-to-end autonomous driving is essentially a continuous action space problem in RL, which can be formulated as a Markov decision process (MDP). MDP is defined by the tuple $\mathcal{M} = \langle \mathcal{S}, \mathcal{A}, \mathcal{P}, \mathcal{R}, \gamma \rangle$, where $\mathcal{S}$ is the state space, $\mathcal{A}$ is the action space, $\mathcal{P}$ is the transition probability, $\mathcal{R}$ is the reward function, and $\gamma$ is the discount factor. The goal in standard RL is to learn a policy $\pi: \mathcal{S} \rightarrow \mathcal{A}$ that maximizes the expected cumulative reward $R_t = \sum_{t=0}^{\infty} \gamma^t r_t$, where $r_t$ is the reward at time $t$. In this study, we employ an entropy-augmented objective function\cite{SAC}, incorporating policy entropy into the reward term:
\begin{equation}
  J_\pi=\underset{\left(s_{i \geq t}, a_{i \geq t}\right) \sim \rho_\pi}{\mathbb{E}}\left[\sum_{i=t}^{\infty} \gamma^{i-t}\left[r_i+\alpha \mathcal{H}\left(\pi\left(\cdot \mid s_i\right)\right)\right]\right],
  \label{entropy_augmented_objective_function}
  \end{equation}
where $\alpha$ is the temperature coefficient, and the policy entropy $\mathcal{H}$ has the form
\begin{equation}
  \mathcal{H}(\pi(\cdot \mid s))=\underset{a \sim \pi(\cdot \mid s)}{\mathbb{E}}[-\log \pi(a \mid s)].
  \end{equation}
\indent The soft \(Q\) value is given by
\begin{equation}
  \begin{aligned}
  Q^\pi&\left(s_t, a_t\right)=r_t\\
  &\ \ \ \ \ +\gamma \underset{\substack{\left(s_{i>t}, a_{i>t}\right) \sim \rho_\pi}}{\mathbb{E}}\left[\sum_{i=t}^{\infty} \gamma^{i-t}\left[r_i-\alpha \log \pi\left(a_i \mid s_i\right)\right]\right],
\end{aligned}
\label{soft_q_value}
\end{equation}
which delineates the expected soft return for choosing $a_t$ at state $s_t$ under policy $\pi$. The soft \(Q\) value can be updated using the soft Bellman operator $\mathcal{T}^\pi$ 
\begin{equation}
  \begin{aligned}
    \mathcal{T}^\pi Q^\pi(s, a)=r+\gamma \underset{s^{\prime} \sim p, a^{\prime} \sim \pi}{\mathbb{E}}\left[Q^\pi\left(s^{\prime}, a^{\prime}\right)-\alpha \log \pi\left(a^{\prime} \mid s^{\prime}\right)\right].
\end{aligned}
\label{soft_policy_evaluation}
\end{equation}
  \indent Meanwhile, the policy $\pi$ is updated by maximizing the entropy-augmented objective (\ref{entropy_augmented_objective_function}), that is,
  \begin{equation}
    \begin{aligned}
    \pi_{\text {new }} & =\arg \max _\pi J_\pi \\
    & =\arg \max _\pi \underset{s \sim \rho_\pi, a \sim \pi}{\mathbb{E}}\left[Q^{\pi_{\text {old }}}(s, a)-\alpha \log \pi(a \mid s)\right].
    \end{aligned}
    \label{soft_policy_improvement}
    \end{equation}
\indent HAC extends the standard RL by integrating human expert into the learning process. During the training process, the human expert monitors the agent-environment interactions, and can intervene the agent's exploration by providing guidance. Specifically, if the agent encounters a risky situation or makes a suboptimal decision, the human expert can overwrite the agent action $a_g$ with their own action $a_h$. The action applied to the environment can thus be expressed as $\widehat{a}  = I\left(s, a\right) a_h+\left(1-I\left(s, a\right)\right) a_g$, where $I\left(s, a_g\right)$ is a Boolean indicator. Thus, with the human policy denoted by $\pi^h$, the shared policy $\pi^b$, which generates the actual trajectory, is defined as 
\begin{equation}
  \pi^b(a \mid s)=\pi^g(a \mid s)(1-I(s, a))+\pi^h(a \mid s) G(s),
  \label{mixed_policy}
  \end{equation}
where $G(s)=\int_{a^{\prime} \in \mathcal{A}} I\left(s, a^{\prime}\right) \pi^g\left(a^{\prime} \mid s\right) d a^{\prime}$ is the probability of human intervention. 

\subsection{Distributional Proxy Value Propagation}

\begin{figure}[]
  \centering
  \subfloat[Human oversees the agent-environment interactions.]{\includegraphics[width=1.35in]{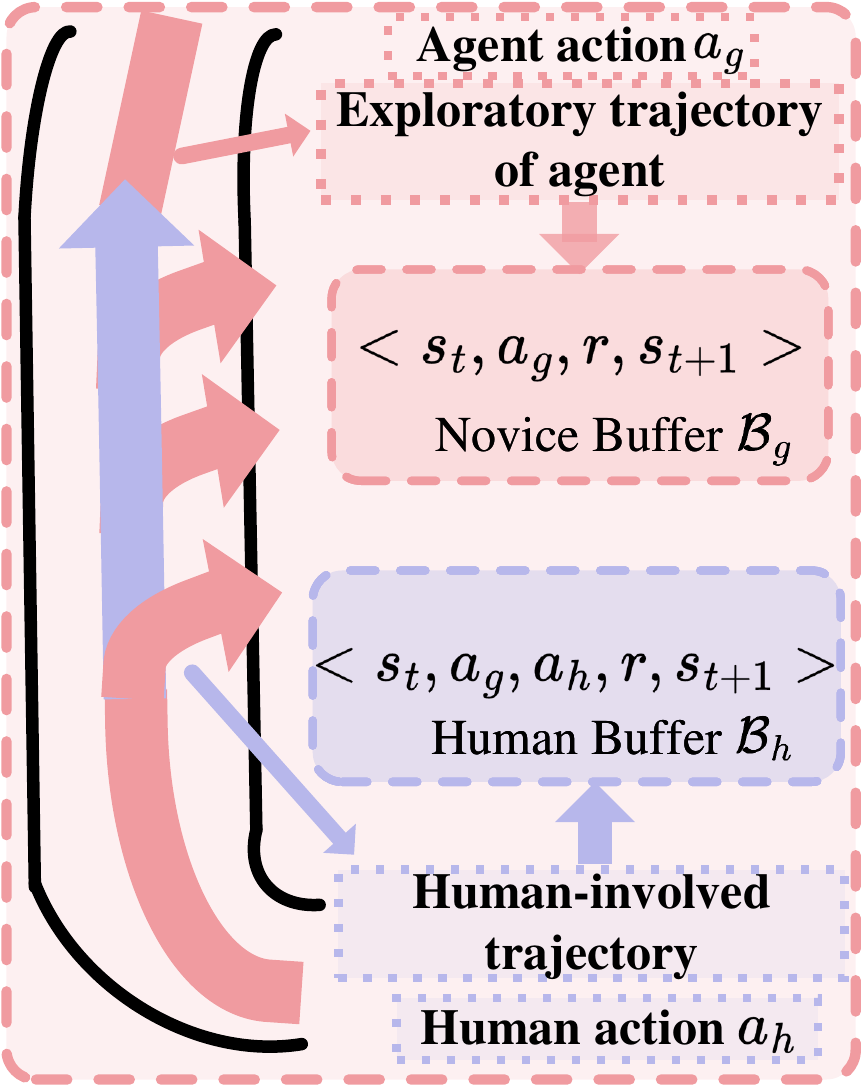}}%
  \hspace{0.1cm}
  \subfloat[Label value distribution through buffers.]{\includegraphics[width=1.85in]{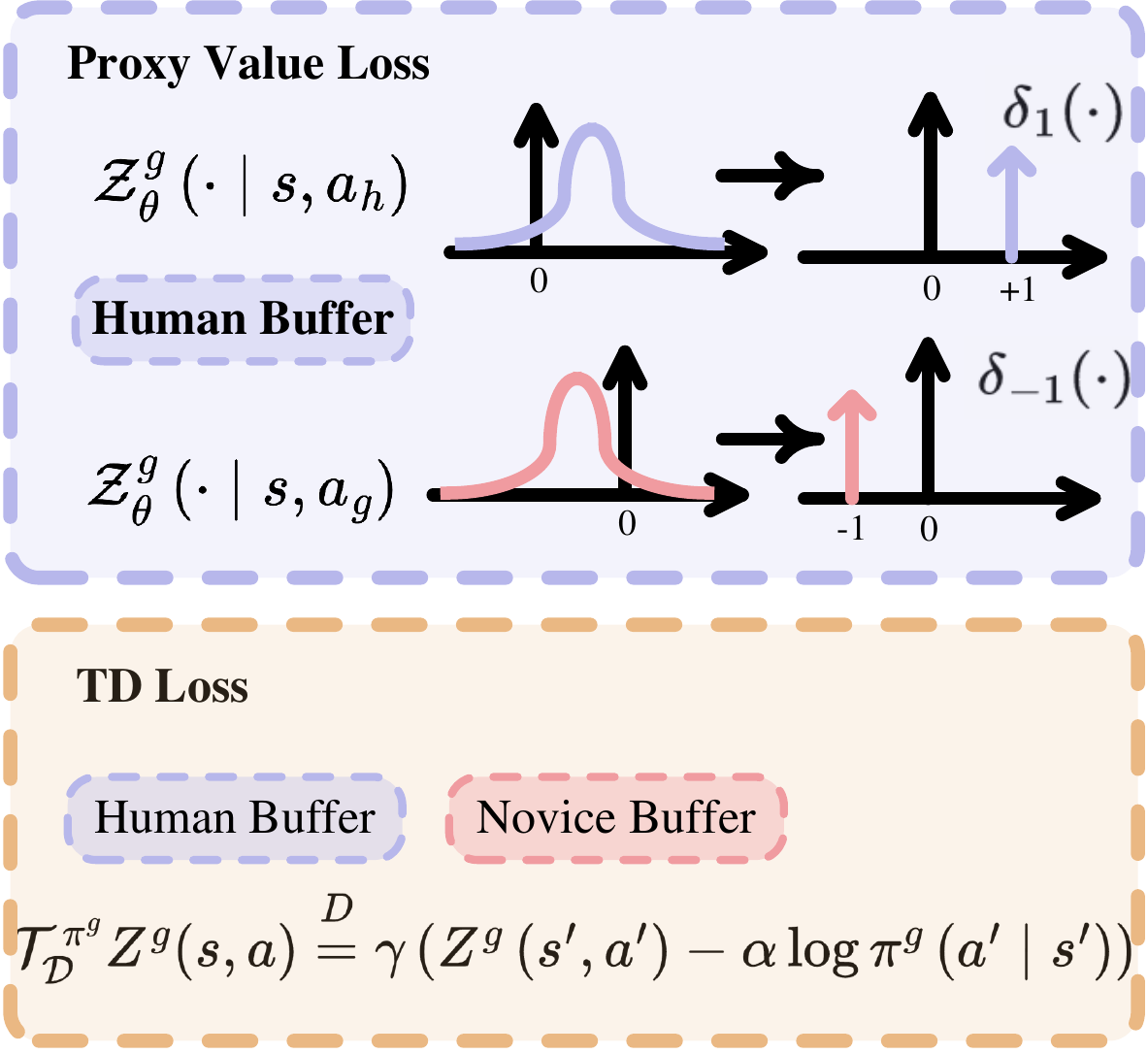}}%
  \caption{Illustration of Distributional Proxy Value Propagation.}
  \label{fig_dpvp}
  \end{figure}

\indent The HAC, which effective in guiding agents to exhibit human-like behaviors, is heavily reliant on human intervention. To alleviate this burden on human operators and enhance training efficiency, the D-PVP is introduced. D-PVP builds upon PVP, which manipulates the \(Q\) value to promote desired behaviors. Specifically, PVP assigns a value of \(+1\) to human actions and \(-1\) to novice actions, thereby encouraging agents to learn human-aligned strategies. The original PVP, as presented in \cite{PVP}, is based on the twin delayed deep deterministic policy gradient (TD3) algorithm, which employs a deterministic policy gradient. To enable a more general and efficient learning process, and to provide an estimate of policy confidence, PVP is extended to the DSAC framework. In this extended version, the distribution of soft state-action returns in DSAC is leveraged to induce the desired behaviors.\\
\indent Firstly, we define the soft state-action return as
\begin{equation}
  Z^\pi\left(s_t, a_t\right):=r_t+\gamma \sum_{i=t}^{\infty} \gamma^{i-t}\left[r_i-\alpha \log \pi\left(a_i \mid s_i\right)\right].
  \end{equation}
\indent From (\ref{soft_q_value}), it follows that $Q^\pi(s, a)=\mathbb{E}\left[Z^\pi(s, a)\right]$. To represent the distribution of the random variable $Z^\pi(s, a)$, let  $\mathcal{Z}^\pi\left(Z^\pi(s, a) \mid s, a\right)$ denote the mapping from $(s, a)$ to a probability distribution over $Z^\pi(s, a)$. Consequently, the distributional version of the soft bellman operator in (\ref{soft_policy_evaluation}) becomes
\begin{equation}
  \mathcal{T}_{\mathcal{D}}^\pi Z(s, a) \stackrel{D}{=} r+\gamma\left(Z\left(s^{\prime}, a^{\prime}\right)-\alpha \log \pi\left(a^{\prime} \mid s^{\prime}\right)\right),
  \label{distributional_version_of_the_soft_Bellman_operator}
  \end{equation}
where $s^{\prime} \sim p, a^{\prime} \sim \pi$, and $A \stackrel{D}{=} B$ indicates that two random variables $A$ and $B$ share identical probability laws. For further usage, a reward-free distributional soft Bellman operator is defined as
\begin{equation}
  \widehat{\mathcal{T}}_{\mathcal{D}}^\pi Z(s, a) \stackrel{D}{=} \gamma\left(Z\left(s^{\prime}, a^{\prime}\right)-\alpha \log \pi\left(a^{\prime} \mid s^{\prime}\right)\right),
  \label{reward_free_distributional_version_of_the_soft_Bellman_operator}
  \end{equation}
and the return distribution is updated via
\begin{equation}
  \mathcal{Z}_{\text {new }}=\arg \min _{\mathcal{Z}} \underset{(s, a) \sim \rho_\pi}{\mathbb{E}}\left[D_{\mathrm{KL}}\left(\mathcal{T}_{\mathcal{D}}^\pi \mathcal{Z}_{\mathrm{old}}(\cdot \mid s, a), \mathcal{Z}(\cdot \mid s, a)\right)\right],
  \end{equation}
where $D_{\mathrm{KL}}$ is the Kullback-Leibler (KL) divergence.\\
\indent D-PVP relies on two value-distribution networks and one stochastic policy, parameterized by $\mathcal{Z}^g_\theta(\cdot \mid s, a)$, $\mathcal{Z}^c_\zeta(\cdot \mid s, a)$ and $\pi^g_\phi(\cdot \mid s)$. The parameters $\theta$, $\zeta$ and $\phi$ govern these networks. Specifically, the distribution $\mathcal{Z}^g_\theta$ is designed for proxy value propagation, while $\mathcal{Z}^c_\zeta$ supports policy confidence evaluation. The policy $\pi^g_\phi$ is guided by human actions. All three networks are modeled as diagonal Gaussian, outputting mean and standard deviation.\\
\indent Fig. \ref{fig_dpvp} illustrates the D-PVP. During training, a human subject supervises the agent-environment interactions. Those exploratory transitions by the agent are stored in the novice buffer $\mathcal{B}_g=\left\{\left(s, a_g, s^{\prime}, r\right)\right\}$. At any time, the human subject can intervene in the agent's free exploration by taking control using the device. During human involvement, both human and novice actions will be recorded into the human buffer $\mathcal{B}_h=\left\{\left(s, a_g, a_h, s^{\prime}, r\right)\right\}$. Meanwhile, the novice policy $\pi^g_\phi(\cdot \mid s)$ is updated following the D-PVP procedure.\\
\indent As illustrated in Fig. \ref{fig_dpvp}(b), to emulate human behavior and minimize intervention, D-PVP samples data $\left(s, a_g, a_h\right)$ from the human buffer and labels the value distribution of the human action $a_h$ with $\delta_{1}(\cdot)$ and the novice action $a_g$ with $\delta_{-1}$. Here $\delta_{1}(\cdot)$ and $\delta_{-1}(\cdot)$ represent the Dirac delta distribution centered at 1 and -1. This labeling fits $\mathcal{Z}^g_\theta(\cdot \mid s, a)$ via the following proxy value (PV) loss: 
\begin{equation}
  J_{\mathcal{Z}^g}^{PV}(\theta)=\left(J_{\mathcal{Z}^g}^H(\theta)+J_{\mathcal{Z}^g}^N(\theta)\right) I\left(s, a_g\right),
\end{equation}
where 
\begin{equation}
J_{\mathcal{Z}^g}^H(\theta)=\underset{(s, a_h,a_g) \sim \mathcal{B}_h}{\mathbb{E}}\left[D_{\mathrm{KL}}\left(\delta_{1}(\cdot), \mathcal{Z}^g_\theta(\cdot \mid s, a_h)\right)\right],
\end{equation}
\begin{equation}
  J_{\mathcal{Z}^g}^N(\theta)=\underset{(s, a_h,a_g) \sim \mathcal{B}_h}{\mathbb{E}}\left[D_{\mathrm{KL}}\left(\delta_{-1}(\cdot), \mathcal{Z}^g_\theta(\cdot \mid s, a_g)\right)\right].
  \end{equation}
\indent Since $\mathcal{Z}^g_\theta(\cdot \mid s, a)$ has the form of Gaussian, it is represented as $\mathcal{N}\left(Q_\theta(s, a), \sigma_\theta(s, a)^2\right)$, where $Q_\theta(s, a)$ and $\sigma_\theta(s, a)$ are the mean and standard deviation of return distribution. The update gradients of $J_{\mathcal{Z}^g}^H(\theta)$ and $J_{\mathcal{Z}^g}^N(\theta)$ are:

\begin{equation}
\begin{aligned}
\nabla_\theta J_{\mathcal{Z}^g}^H(\theta)=&\mathbb{E}\left[\nabla_\theta \frac{\left(1-Q_\theta(s, a_h)\right)^2}{2 \sigma_\theta(s, a_h)^2}+\eta \frac{\nabla_\theta \sigma_\theta(s, a_h)}{\sigma_\theta(s, a_h)}\right] \\
=&\mathbb{E}[-\frac{\left(1-Q_\theta(s, a_h)\right)}{\sigma_\theta(s, a_h)^2} \nabla_\theta Q_\theta(s, a_h) \\
&-\frac{\left(1-Q_\theta(s, a_h)\right)^2-\sigma_\theta(s, a_h)^2}{\sigma_\theta(s, a_h)^3} \eta \nabla_\theta \sigma_\theta(s, a_h)],
\end{aligned}
\end{equation}

  and
  \begin{equation}
    \begin{aligned}
    \nabla_\theta J_{\mathcal{Z}^g}^N(\theta)=&\mathbb{E}\left[\nabla_\theta \frac{\left(-1-Q_\theta(s, a_g)\right)^2}{2 \sigma_\theta(s, a_g)^2}+\eta \frac{\nabla_\theta \sigma_\theta(s, a_g)}{\sigma_\theta(s, a_g)}\right] \\
    =&\mathbb{E}\left[\frac{\left(1+Q_\theta(s, a_g)\right)}{\sigma_\theta(s, a_g)^2} \nabla_\theta Q_\theta(s, a_g)\right.\\
&\left.-\frac{\left(1+Q_\theta(s, a_g)\right)^2-\sigma_\theta(s, a_g)^2}{\sigma_\theta(s, a_g)^3} \eta \nabla_\theta \sigma_\theta(s, a_g)\right],
    \end{aligned}
    \end{equation}
here $\eta$ modulates the variance convergence rate.\\
\indent The transitions stored in the novice buffer, though devoid of direct human intervention, still encapsulate valuable insights regarding forward dynamics and human preferences. Rather than discarding these data, D-PVP propagates proxy values to these states through a reward-free TD update. As depicted in Fig. \ref{fig_dpvp}(b), the reward-free TD loss admits
\begin{equation}
  J^{TD}_{\mathcal{Z}^g}(\theta)=\underset{(s, a) \sim \mathcal{B}}{\mathbb{E}}\left[D_{\mathrm{KL}}\left(\widehat{\mathcal{T}}_{\mathcal{D}}^{\pi^g_{\bar{\Phi}}} \mathcal{Z}^g_{\bar{\theta}}(\cdot \mid s, a), \mathcal{Z}^g_\theta(\cdot \mid s, a)\right)\right],
  \label{reward_free_td_loss}
  \end{equation}
where $\bar{\theta}$ and $\bar{\phi}$ are the target-network parameters, and $\mathcal{B}$ refers to $\mathcal{B}_g\cup \mathcal{B}_h$. Since $\widehat{\mathcal{T}}_{\mathcal{D}}^{\pi^g_{\bar{\Phi}}} \mathcal{Z}^g_{\bar{\theta}}(\cdot \mid s, a)$ is unknown, a sample-based version of (\ref{reward_free_td_loss}) is applied:
\begin{equation}
  J^{TD}_{\mathcal{Z}^g}(\theta)=-\underset{\substack{\left(s, a, s^{\prime}\right) \sim \mathcal{B}, a^{\prime} \sim \pi^g_{\bar{\Phi}},\\ Z^g\left(s^{\prime}, a^{\prime}\right) \sim \mathcal{Z}^g_{\bar{\theta}}\left(\cdot \mid s^{\prime}, a^{\prime}\right)}}{\mathbb{E}}\left[\log \mathcal{P}\left(\widehat{y}_z \mid \mathcal{Z}^g_\theta(\cdot \mid s, a)\right)\right],
  \label{kl_update}
  \end{equation}
with the reward-free target value
\begin{equation}
  \widehat{y}_z=\gamma\left(Z^g\left(s^{\prime}, a^{\prime}\right)-\alpha \log \pi^g_{\bar{\phi}}\left(a^{\prime} \mid s^{\prime}\right)\right).
  \end{equation}
\indent The corresponding update gradient is
\begin{equation}
  \begin{aligned}
  \nabla_\theta J_{\mathcal{Z}^g}^{TD}(\theta)=&\mathbb{E}\left[\nabla_\theta \frac{\left(\widehat{y}_z-Q_\theta(s, a)\right)^2}{2 \sigma_\theta(s, a)^2}+\eta \frac{\nabla_\theta \sigma_\theta(s, a)}{\sigma_\theta(s, a)}\right] \\
  =&\mathbb{E}[-\frac{\left(\widehat{y}_z-Q_\theta(s, a)\right)}{\sigma_\theta(s, a)^2} \nabla_\theta Q_\theta(s, a) \\
  &-\frac{\left(\widehat{y}_z-Q_\theta(s, a)\right)^2-\sigma_\theta(s, a)^2}{\sigma_\theta(s, a)^3} \eta \nabla_\theta \sigma_\theta(s, a)].
  \end{aligned}
  \end{equation}
  \indent Hence, the final value loss for $\mathcal{Z}^g_\theta(\cdot \mid s, a)$ integrates both PV loss and TD loss:
\begin{equation}
J_{\mathcal{Z}^g}(\theta)=J_{\mathcal{Z}^g}^{PV}(\theta)+J_{\mathcal{Z}^g}^{TD}(\theta).
\label{final_value_loss}
\end{equation}
\indent For policy improvement, the actor  $\pi^g_\phi(\cdot \mid s)$ is optimized by maximizing the return distribution
\begin{equation}
\begin{aligned}
J_{\pi^g}(\phi) & =\underset{\substack{s \sim \mathcal{B},\\ a \sim \pi^g_\phi}}{\mathbb{E}}\left[\underset{Z^g(s, a) \sim \mathcal{Z}^g_\theta(\cdot \mid s, a)}{\mathbb{E}}[Z^g(s, a)]-\alpha \log \left(\pi^g_\phi(a \mid s)\right)\right] \\
& =\underset{s \sim \mathcal{B}, a \sim \pi_\phi^g}{\mathbb{E}}\left[Q_\theta(s, a)-\alpha \log \left(\pi^g_\phi(a \mid s)\right)\right],
\end{aligned}
\label{policy_improvement}
\end{equation}
while the temperature $\alpha$ is updated to balance exploration and exploitation:
\begin{equation}
\alpha \leftarrow \alpha- \mathbb{E}_{s \sim \mathcal{B}, a \sim \pi^g_\phi}\left[-\log \pi^g_\phi(a \mid s)-\overline{\mathcal{H}}\right].
\end{equation}

\indent The expected cumulative probability of failure $V_{\pi^b}$ of the behavior policy $\pi^b$ in D-PVP can be upper-bounded by:
\begin{equation}
\label{eq:risk_bound}
    V_{\pi^b} \;\le\; \frac{1}{1-\gamma}
    \Bigl(
      \epsilon + \kappa \;+\; \frac{\gamma\,\epsilon^2}{1-\gamma}\,K'
    \Bigr),
\end{equation}
where $\epsilon < 1$ and $\kappa < 1$ are small constants representing, respectively, the probability that the human expert provides an unsafe action and the probability that the expert fails to intervene when the agent takes an unsafe action. The term $K' \ge 0$ captures the human expert's tolerance, corresponding to the measure of the action set in which the expert can intervene. For further details and the derivation of this upper bound, please refer to \cite{HAIM}.
\subsection{Shared Control Mechanism with Policy Confidence Evaluation}

\indent Exclusive reliance on human guidance can lead to suboptimal policy outcomes. Therefore, it is essential to enable agents to explore independently and continuously improve their performance beyond human guidance. This can be achieved by incorporating a reward signal into the learning process, allowing the agent to learn from the environment. However, directly adding the reward signal throughout the entire training can destabilize the learning process due to discrepancies between proxy values and reward signals. Alternatively, introducing the reward signal after the agent has learned a basic driving policy from human demonstrations may result in significant performance degradation. This occurs because the native reward function might not align with human preferences, hindering the retention of previously learned human policies.\\
\indent To address these challenges, our C-HAC framework puts forward a shared control mechanism that combines the learned human-guided policy with a self-learning policy. The shared control mechanism is defined as:
\begin{equation}
  \pi^b\left( a \mid s\right) =  \pi^r_\psi\left( a \mid s\right)\left( 1-\mathcal{T}_{c}(s) \right) + \pi^g_\phi\left( a \mid s\right) \mathcal{T}_{c}(s) .
  \label{shared_control_mechanism}
  \end{equation}
\indent Here, $\mathcal{T}_{\text {c}}$ denotes the confidence-based intervention function. The policy $\pi^g_\phi\left( s, a\right)$ represents the human-guided policy learned from human demonstrations through D-PVP, while $\pi^r_\psi\left( a \mid s\right)$  is the self-learning policy derived from the human-guided policy but optimized to maximize cumulative rewards. The self-learning policy is improved by maximizing the return distribution:
\begin{equation}
  \begin{aligned}
  J_{\pi^r}(\psi) & =\underset{s \sim \mathcal{B}, a \sim \pi_\psi^r}{\mathbb{E}}\bigg[\underset{Z^c(s, a) \sim \mathcal{Z}^c_\zeta(\cdot \mid s, a)}{\mathbb{E}}[Z^c(s, a)]\\
  &\ \ \ \ \ \ \ \ \ \ \ \ \ \ \ \ \ \ \ \ \ \ \ \ \ \ \ -\alpha \log \left(\pi_\psi^r(a \mid s)\right)\bigg] \\
  & =\underset{s \sim \mathcal{B}, a \sim \pi_\psi^r}{\mathbb{E}}\left[Q_\zeta(s, a)-\alpha \log \left(\pi_\psi^r(a \mid s)\right)\right].
  \end{aligned}
  \label{pi_r_loss}
  \end{equation}
\indent The distribution $\mathcal{Z}^c_\zeta(\cdot \mid s, a)$ is updated using the sample-based TD loss:
\begin{equation}
  J^{TD}_{\mathcal{Z}^c}(\zeta)=-\underset{\substack{\left(s, a, r, s^{\prime}\right) \sim \mathcal{B}, a^{\prime} \sim \pi^r_{\bar{\psi}},\\ Z^c\left(s^{\prime}, a^{\prime}\right) \sim \mathcal{Z}^c_{\bar{\zeta}}\left(\cdot \mid s^{\prime}, a^{\prime}\right)}}{\mathbb{E}}\left[\log \mathcal{P}\left(y_z^r \mid \mathcal{Z}^c_\zeta(\cdot \mid s, a)\right)\right],
  \label{Z_c_td_loss}
  \end{equation}
  where the target value $y_z^r$ is defined as:
  \begin{equation}
    y_z^r=r+\gamma\left(Z^c\left(s^{\prime}, a^{\prime}\right)-\alpha \log \pi^r_{\bar{\psi}}\left(a^{\prime} \mid s^{\prime}\right)\right).
    \label{value_target}
    \end{equation}
\indent This mechanism enables the agent to exploit potentially superior reward-based strategies while retaining the essential safety and efficiency derived from human demonstrations. A critical component of this mechanism is the policy confidence evaluation, that determines which policy to follow at each step. The confidence evaluation utilizes the return distribution $\mathcal{Z}^c_\zeta\left(s, a\right)$. By updating with the value target in (\ref{value_target}), $\mathcal{Z}^c_\zeta\left(s, a\right)$ estimates the distribution of cumulative returns given states and actions. Within the shared control mechanism, for each state-action pair $\left(s, a_g\right)$ associated with the human-guided policy and $\left(s, a_r\right)$ associated with the self-learning policy, $\mathcal{Z}^c_\zeta\left(s, a\right)$ outputs the distributions of cumulative rewards as:
\begin{equation}
  \begin{aligned}
  & \mathcal{Z}^c_\zeta\left(s, a_r\right) \rightarrow \mathcal{N}\left(Q_\zeta\left(s, a_r\right), \sigma_\zeta\left(s, a_r\right)^2\right), \\
  & \mathcal{Z}^c_\zeta\left(s, a_g\right) \rightarrow \mathcal{N}\left(Q_\zeta\left(s, a_g\right), \sigma_\zeta\left(s, a_g\right)^2\right).
  \end{aligned}
  \end{equation}
  \indent Here, $Q_\zeta\left(s, a_r\right)$ and $Q_\zeta\left(s, a_g\right)$ represent the expected cumulative returns for the reward-guided and human-guided policies, respectively. The standard deviations $\sigma_\zeta\left(s, a_r\right)$ and $\sigma_\zeta\left(s, a_g\right)$ indicate the uncertainties in the cumulative reward estimates for each policy. For simplicity, $Q_\zeta\left(s, a_r\right)$ and $Q_\zeta\left(s, a_g\right)$ are denoted as $Q_\zeta^r$ and $Q_\zeta^g$, respectively. Similarly, $\sigma_\zeta\left(s, a_r\right)$ and $\sigma_\zeta\left(s, a_g\right)$ are denoted as $\sigma_\zeta^r$ and $\sigma_\zeta^g$\\
  \indent Using these outputs, the confidence that the self-learning policy outperforms the human-guided policy in state $s$ is calculated as:
  \begin{equation}
    P\left(Q_\zeta^r>Q_\zeta^g\right)=1-\Phi\left(\frac{Q_\zeta^g-Q_\zeta^r}{\sqrt{\left(\sigma_\zeta^r\right)^2+\left(\sigma_\zeta^g\right)^2}}\right),
    \end{equation}
where $\Phi(\cdot)$ is the cumulative distribution function (CDF) of the standard normal distribution.\\
\indent Based on the confidence evaluation, the confidence-based intervention function $\mathcal{T}_{c}$ determines which policy to follow in a given state $s$:
\begin{equation}
  \mathcal{T}_{c}(s)= \begin{cases}1 & \text { if } P\left(Q_\zeta^{r}>Q_\zeta^{g}\right) \leq  1-\delta, \\ 0 & \text { otherwise. }\end{cases}
  \label{confidence intervention}
  \end{equation}
\indent If the confidence probability exceeds the threshold $1-\delta$, the reward-guided policy $\pi^r_\psi$ will be selected. Otherwise, the human-guided policy $\pi^g_\phi$ is used.
\begin{assumption}
  (Bounded variance of return distributions.)  To simplify the analysis, the variances of the return distributions for both the human-guided policy ($\pi^g$) and the reward-guided policy ($\pi^r$) are bounded by a constant $\sigma_{\max}$. Specifically:
$$
\sigma_g^2 \leq \sigma_{\max }^2, \quad \sigma_r^2 \leq \sigma_{\max }^2.
$$
\end{assumption}
\begin{theorem}
With the confidence-based intervention function $\mathcal{T}_{\text {c}}(s)$ defined in (\ref{confidence intervention}), the return of the behavior policy $\pi^b$ is guaranteed to be no worse than:
\begin{equation}
J\left(\pi^b\right) \geq J\left(\pi^g\right)-(1-\beta) \cdot \frac{\sqrt{2} \cdot \sigma_{\max } \cdot \Phi^{-1}(\delta)}{1-\gamma},
\end{equation}
where $\beta$ is the expected intervention rate weighted by the policy discrepancy.
\end{theorem}
\begin{proof}
  To prove Theorem 1, several useful lemmas are introduced.
\begin{lemma}
  For behavior policy $\pi^b$ deduced by human-guided policy $\pi^g$, reward-guided policy $\pi^r$ and an intervention function $\mathcal{T}_c(s)$, the state distribution discrepancy between $\pi^b$ and $\pi^r$ is bounded by (Theorem 3.2 in \cite{lemma2})
  \begin{equation}
    \left\|d_{\pi^b}-d_{\pi^r}\right\|_1 \leq \frac{\beta \gamma}{1-\gamma} \mathbb{E}_{s \sim d_{\pi^b}}\left\|\pi^g(\cdot \mid s)-\pi^r(\cdot \mid s)\right\|_1,
    \end{equation}
    where $\beta=\frac{\mathbb{E}_{s \sim d_{\pi^b}}\left[\mathcal{T}(s)\left\|\pi^g(\cdot \mid s)-\pi^r(\cdot \mid s)\right\|_1\right]}{\mathbb{E}_{s \sim d_{\pi^b}}\left\|\pi^g(\cdot \mid s)-\pi^r(\cdot \mid s)\right\|_1} \in[0,1]$ is the expected intervention rate weighted by the policy discrepancy.
\end{lemma}

\begin{lemma}
  The performance difference between two policies in terms of the advantage function can be expressed as (Equation 3 in \cite{trpo}):
  \begin{equation}
    J(\pi)=J\left(\pi^{\prime}\right)+\mathbb{E}_{s_t, a_t \sim \tau_\pi}\left[\sum_{t=0}^{\infty} \gamma^t A_{\pi^{\prime}}\left(s_t, a_t\right)\right].
    \end{equation}
\end{lemma}
\indent With Lemma 2, the proof of Theorem 2 is given as follows.
\begin{equation}
  \begin{aligned}
    & J\left(\pi^b\right)-J\left(\pi^g\right)\\ 
    & =\underset{s_n, a_n \sim \tau_{\pi^b}}{\mathbb{E}}\left[\sum_{n=0}^{\infty} \gamma^n A^g\left(s_n, a_n\right)\right] \\
    & =\underset{s_n, a_n \sim \tau_{\pi^b}}{\mathbb{E}}\left[\sum_{n=0}^{\infty} \gamma^n\left[Q^g\left(s_n, a_n\right)-V^g\left(s_n\right)\right]\right] \\
    & =\underset{s_n\sim \tau_{\pi^b}}{\mathbb{E}}\left[\sum_{n=0}^{\infty} \gamma^n\left[\mathbb{E}_{a \sim \pi^b\left(\cdot \mid s_n\right)} Q^g\left(s_n, a\right)-V^g\left(s_n\right)\right]\right] \\
    & =\underset{s_n\sim \tau_{\pi^b}}{\mathbb{E}}\bigg[\sum_{n=0}^{\infty} \gamma^n\left[\mathcal{T}_c\left(s_n\right) \mathbb{E}_{a \sim \pi^g\left(\cdot \mid s_n\right)} Q^g\left(s_n, a\right) \right. \\
    & \quad\quad\quad\left. + \left(1-\mathcal{T}_c\left(s_n\right)\right) \mathbb{E}_{a \sim \pi^r\left(\cdot \mid s_n\right)} Q^g\left(s_n, a\right)-V^g\left(s_n\right)\right]\bigg] \\
    & =\underset{s_n\sim \tau_{\pi^b}}{\mathbb{E}}\bigg[\sum_{n=0}^{\infty} \gamma^n\left[\left(1-\mathcal{T}_c\left(s_n\right)\right)\left[\mathbb{E}_{a \sim \pi^r\left(\cdot \mid s_n\right)} Q^g\left(s_n, a\right)\right.\right.\\
    &\quad\quad\quad\quad\quad\quad\quad\quad\quad\quad\quad\quad\quad\quad\quad\quad\quad\quad\left.\left.-V^g\left(s_n\right)\right]\right]\bigg] \\
    & =(1-\beta) \underset{s_n\sim \tau_{\pi^b}}{\mathbb{E}}\bigg[\sum_{n=0}^{\infty} \gamma^n\left[\mathbb{E}_{a \sim \pi^r\left(\cdot \mid s_n\right)} Q^g\left(s_n, a\right)\right.\\
    &\quad\quad\quad\quad\quad\quad\quad\quad\quad\quad\quad\quad\quad\quad\quad\quad\quad\quad\left.-V^g\left(s_n\right)\right]\bigg].
  \end{aligned}
\end{equation}
\indent Note that the choice of $\pi^r$ is guided by the confidence probability $P\left(Q^g\left(s, a_r\right)>Q^g\left(s, a_g\right)\right)$. When the human-guided policy is selected, its expected advantage over the human-guided policy can be expressed as:

\begin{equation}
\begin{aligned}
\mathbb{E}_{a \sim \pi^r\left(\cdot \mid s_n\right)}& Q^g\left(s_n, a\right)-V^g\left(s_n\right)\\
&=\mathbb{E}_{a \sim \pi^r\left(\cdot \mid s_n\right)} Q^g\left(s_n, a\right)-\mathbb{E}_{a \sim \pi^g\left(\cdot \mid s_n\right)} Q^g\left(s_n, a\right).
\end{aligned}
\end{equation}

Using the confidence-based condition for selecting human-guided policy:

$$
P\left(Q^g\left(s, a_r\right)>Q^g\left(s, a_g\right)\right) \leq 1-\delta,
$$
we have:
\begin{equation}
\begin{aligned}
\mathbb{E}_{a \sim \pi^r\left(\cdot \mid s_n\right)} Q^g\left(s_n, a\right)-\mathbb{E}_{a \sim \pi^g\left(\cdot \mid s_n\right)} &Q^g\left(s_n, a\right)\\
&\leq \sqrt{\sigma_r^2+\sigma_g^2} \cdot \Phi^{-1}(\delta),
\end{aligned}
\end{equation}
where $\Phi^{-1}(\delta)$ is the inverse CDF of the standard normal distribution, and $\sigma_r^2, \sigma_g^2$ are the variances of the value-distribution for $\pi^r$ and $\pi^g$, respectively.\\
\indent Substituting the lower bound of the Q-value difference into the performance difference yields

\begin{equation}
\begin{aligned}
J\left(\pi^b\right)-&J\left(\pi^g\right) \\
&\geq -(1-\beta) \mathbb{E}_{s_n \sim \tau_{\pi^b}}\left[\sum_{n=0}^{\infty} \gamma^n \cdot \sqrt{\sigma_r^2+\sigma_h^2} \cdot \Phi^{-1}(\delta)\right].
\end{aligned}
\end{equation}
\indent By applying the geometric series formula and variances bound of the value-distribution, the performance difference becomes:

$$
J\left(\pi^b\right)-J\left(\pi^g\right) \geq -(1-\beta) \cdot \frac{\sqrt{2} \cdot \sigma_{\max } \cdot \Phi^{-1}(\delta)}{1-\gamma} .
$$

Thus, the final lower bound for the mixed policy satisfies

$$
J\left(\pi^b\right) \geq J\left(\pi^g\right)-(1-\beta) \cdot \frac{\sqrt{2} \cdot \sigma_{\max } \cdot \Phi^{-1}(\delta)}{1-\gamma}.
$$
\end{proof}
\subsection{Overall Training Process}
During the human demonstration stage, the agent employs two value distribution networks, $\mathcal{Z}_\theta^g(\cdot \mid s, a)$ and $\mathcal{Z}^c_\zeta(\cdot \mid s, a)$, as well as the policy network $\pi^g_\phi(\cdot \mid s)$. Specifically, in each training iteration, two equally-sized batches, $b_g$ and $b_h$, are sampled from $\mathcal{B}_g$ and $\mathcal{B}_h$, respectively. The agent updates the value distribution network $\mathcal{Z}^g_\theta(\cdot \mid s, a)$ and the policy network $\pi_\phi^g(\cdot \mid s)$ by minimizing the loss functions in (\ref{final_value_loss}) and (\ref{policy_improvement}). As mentioned earlier, the update of $\mathcal{Z}^g_\theta(\cdot \mid s, a)$ is reward-free. Concurrently, the value distribution network $\mathcal{Z}_\zeta^c(\cdot \mid s, a)$ is updated by minimizing the loss:
\begin{equation}
J^{TD}_{\mathcal{Z}^c}(\zeta)=-\underset{\substack{\left(s, a_g, r, s^{\prime}\right) \sim \mathcal{B}_g, a^{\prime} \sim \pi^g_{\bar{\phi}},\\ Z^c\left(s^{\prime}, a^{\prime}\right) \sim \mathcal{Z}^c_{\bar{\zeta}}\left(\cdot \mid s^{\prime}, a^{\prime}\right)}}{\mathbb{E}}\left[\log \mathcal{P}\left(y_z^g \mid \mathcal{Z}^c_\zeta(\cdot \mid s, a)\right)\right],
\end{equation}
where the target value $y_z^g$ is defined as:
\begin{equation}
  y_z^g=r+\gamma\left(Z^c\left(s^{\prime}, a^{\prime}\right)-\alpha \log \pi^g_{\bar{\phi}}\left(a^{\prime} \mid s^{\prime}\right)\right).
  \label{confidence_value_target}
  \end{equation}
\indent Once the agent has developed sufficient confidence in its policy evaluation capability and has learned a basic driving policy, it transitions to the RL continuous enhancement stage. The transition conditions are as follows:\\
\begin{equation}
  \left\{\begin{array}{l}\sigma_{\text {mean }}^c<\vartheta_c,
    \\ \log \left(\pi^g_\phi(a \mid s)\right)<\kappa,
    \\ n_{\text {steps }}>N_g.
    \end{array}\right.
\end{equation}
\indent Here, $\sigma_{\text{mean}}^c < \vartheta_c$ indicates that the variance of the policy evaluation network has dropped below a predefined threshold, signifying sufficient confidence in policy evaluation. The condition $\log \pi^g_\phi(a \mid s) < \kappa$ implies that the agent has learned a basic driving policy. $N_g$ is a predefined threshold that specifies the minimum number of training steps required before the agent can transition to the second stage.\\
\indent Upon entering the RL continuous enhancement stage, the self-learning policy $\pi_\psi^r(a \mid s)$ is initialized with the weight copied from the human-guided policy $\pi_\phi^g(a \mid s)$. The share control mechanism in (\ref{shared_control_mechanism}) is activated. During the training process, the agent samples a batch $b_g$ from the novice buffer $\mathcal{B}_g$ and updates the value distribution network $\mathcal{Z}^c_\zeta(\cdot \mid s, a)$ and the policy network $\pi_\psi^r(a \mid s)$ by minimizing the loss functions in (\ref{Z_c_td_loss}) and (\ref{pi_r_loss}), respectively. The temperature parameter $\alpha$ is updated according to (\ref{soft_policy_improvement}). And the confidence-based intervention function $\mathcal{T}_{\text {c}}(s)$ is calculated using (\ref{confidence intervention}). The agent then alternates between human-guided policy and self-learning policy until the training process is complete.\\
\indent Algorithm \ref{alg:chac} gives the details steps for the proposed C-HAC approach.\\
\begin{algorithm}[!ht]
  \caption{Confidence-guided human-AI collaboration (C-HAC)}
  \label{alg:chac}
  \begin{algorithmic}[1]
      \STATE \textbf{Initialize:} $\pi^g_\phi$, $\pi^r_\psi$, $\mathcal{Z}^g_\theta$, $\mathcal{Z}^c_\zeta$, $\mathcal{B}_g$, $\mathcal{B}_h$
      \STATE \textbf{Set:} Temperature parameter $\alpha$
      
      \STATE \textbf{Stage 1: Human Demonstration Learning}
      \WHILE{Not ready to transition}
          \STATE Execute action $a_g \sim \pi^g_\phi(\cdot|s)$ in environment
          \STATE Observe next state $s'$, reward $r$
          \STATE Store transition $(s, a_g, s', r)$ in $\mathcal{B}_g$
          
          \IF{Human intervention occurs}
              \STATE Store $(s, a_g, a_h, s', r)$ in $\mathcal{B}_h$
          \ENDIF
          
          \STATE Sample mini-batches $b_g \sim \mathcal{B}_g$, $b_h \sim \mathcal{B}_h$
          \STATE Update $\mathcal{Z}^g_\theta$ via Proxy Value Loss (\ref{final_value_loss})
          \STATE Update $\pi^g_\phi$ via Policy Improvement (\ref{policy_improvement})
          \STATE Update $\mathcal{Z}^c_\zeta$ via TD Loss (\ref{confidence_value_target})
      \ENDWHILE
      
      \STATE \textbf{Stage 2: RL Continuous Enhancement}
      \STATE Initialize $\pi^r_\psi$ with parameters from $\pi^g_\phi$
      
      \WHILE{Training is not complete}
          \STATE \textbf{Action Selection:}
          \STATE \quad With probability $1 - \mathcal{T}_c(s)$, choose $a \sim \pi^r_\psi(\cdot|s)$
          \STATE \quad With probability $\mathcal{T}_c(s)$, choose $a \sim \pi^g_\phi(\cdot|s)$
          
          \STATE Execute action $a$, observe $s'$, $r$
          \STATE Store transition $(s, a, s', r)$ in $\mathcal{B}_g$
          
          \STATE Sample mini-batch $b_g \sim \mathcal{B}_g$
          \STATE Update $\mathcal{Z}^c_\zeta$ via TD Loss (\ref{Z_c_td_loss})
          \STATE Update $\pi^r_\psi$ via Return Maximization (\ref{pi_r_loss})
          \STATE Update temperature $\alpha$ via (\ref{soft_policy_improvement})
          \STATE Compute $\mathcal{T}_c(s)$ via Confidence Intervention (\ref{confidence intervention})
      \ENDWHILE
  \end{algorithmic}
  \end{algorithm}
\section{SIMULATION VERIFICATION}
\subsection{Experimental Setting}
\begin{figure}[h]
  \centering
  \subfloat{\includegraphics[width=3.4in]{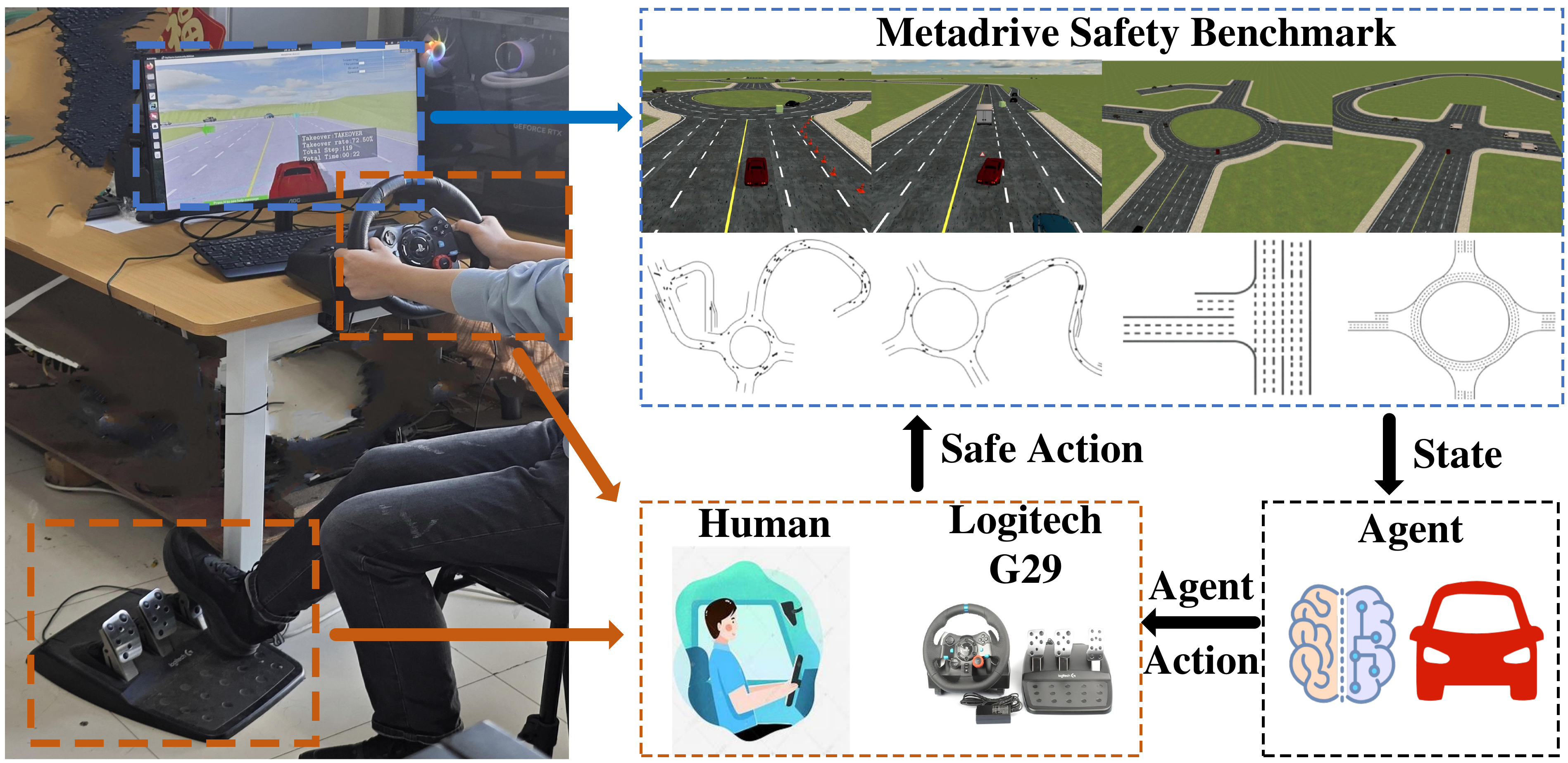}}%
  \caption{Simulation environment and human interfaces.}
  \label{env}
  \end{figure}
To evaluate the proposed C-HAC approach, we conduct experiments using the Metadrive safety benchmark\cite{metadrive}. We generate a diverse set of driving scenarios with each training session consisting of 20 different scenarios. As shown in Fig \ref{env}, the map of each scenario is composed of various typical block types, including straight, ramp, roundabout, T-intersection, and intersection. Each scenario also includes randomly placed obstacles, such as moving traffic vehicles, stationary traffic cones, and triangular warning signs. The definitions of the observation space, action space, environmental reward, and environmental cost are as follows: \\
\indent \textbf{Observation Space}: The observation space is defined as a continuous space with the following elements:
\begin{itemize}
    \item Current state: including the target vehicle's steering, heading, and velocity.
    \item Surrounding information: represented by a vector of 240 LIDAR-like distance measurements from nearby vehicles and obstacles.
    \item Navigation data: including the relative positions toward future checkpoints and the destination.
\end{itemize}
\indent \indent  \textbf{Action Space}: The action space is defined as a continuous space with the acceleration and the steering angle.\\
\indent \textbf{Reward}: The reward function consists of four parts:
\begin{equation}
  R = c_{\text{disp}} R_{\text{disp}} + c_{\text{speed}} R_{\text{speed}} + c_{\text{collision}} R_{\text{collision}} + R_{\text{term}}
\end{equation}
\begin{itemize}
    \item \textbf{$R_{\text{disp}}$}: Encourages forward movement, defined as $R_{\text{disp}} = d_t - d_{t-1}$, where $d_t$ and $d_{t-1}$ are the longitudinal movements at the current and previous time steps.
    \item \textbf{$R_{\text{speed}}$}: Encourages maintaining a reasonable speed, defined as $R_{\text{speed}} = v_t / v_{\max}$, where $v_t$ and $v_{\max}$ denote the current speed and maximum allowed speed.
    \item \textbf{$R_{\text{collision}}$}: Penalizes collisions, defined as $R_{\text{collision}} = -5$ if a collision with a vehicle, human, or object occurs, otherwise it is 0.
    \item \textbf{$R_{\text{term}}$}: This reward is assigned only at the last time step. If the vehicle reaches the destination, we choose $R_{\text{term}} = +10$ (success). If the vehicle drives out of the road, $R_{\text{term}} = -5$.
\end{itemize}

\indent \textbf{Cost}: Each collision with traffic vehicles, obstacles, or parked vehicles incurs a cost of -1. The environmental cost is utilized for testing the safety of the trained policies and measuring the occurrence of dangerous situations during the training process.\\
\indent As shown in Fig \ref{env}, human subjects can take over control using the Logitech G29 racing wheel and monitor the training process through the visualization of environments on the screen. The user interface displays several key indicators: '\textit{Speed}' flag shows the real-time speed of the target vehicles. '\textit{Takeover}' flag is triggered when a human takeover occurs. '\textit{Total step}' flag indicates the number of steps taking during training. '\textit{Total time}' flag represents the total time spent on training. '\textit{Takeover rate}' displays the frequency of human intervention. '\textit{Stage}' flag indicates the current training stage. '\textit{Reward policy}' flag is triggered when the self-learning policy is adopted instead of the human-guided policy.
\begin{figure*}[!t]
  \centering
  \subfloat[Test-time reward comparison.]{\includegraphics[width=2.73in]{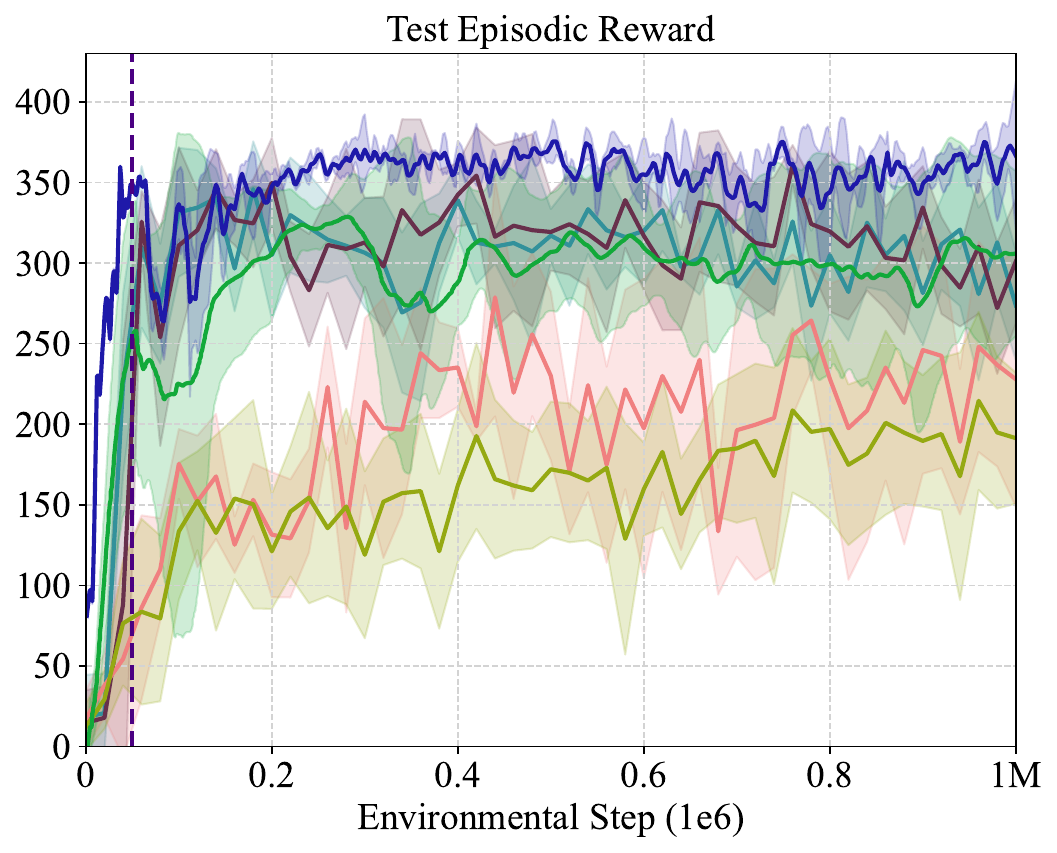}%
  \label{fig_6_1}}
  \hfil
  \subfloat[Train-time cost comparison.]{\includegraphics[width=2.87in]{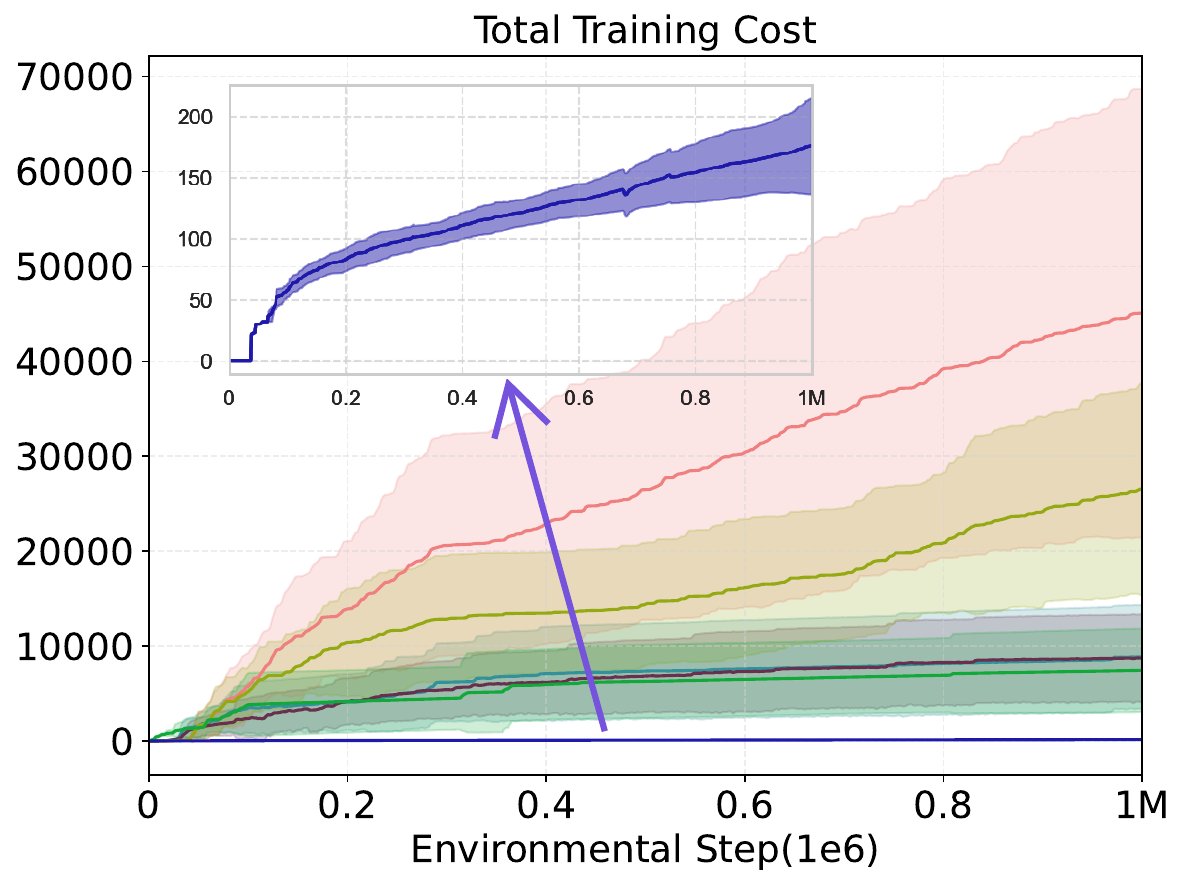}%
  \label{fig_4_2}}
  \hfil
  \begin{minipage}[b]{1.5in} 
    \centering
    \subfloat[Takeover rate comparison.]{\includegraphics[width=\textwidth]{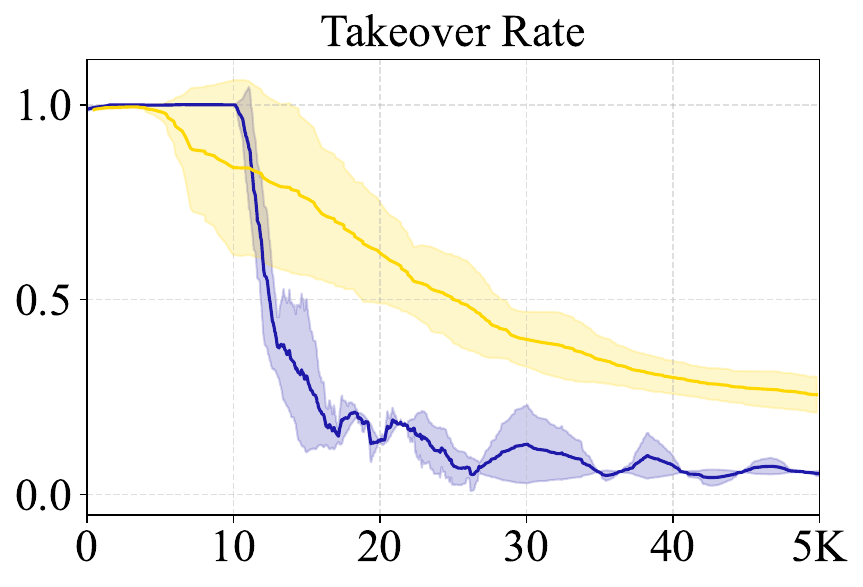}%
    \label{fig_4_3}}\\[0.1em]
    \subfloat[Test-time reward comparison.]{\includegraphics[width=\textwidth]{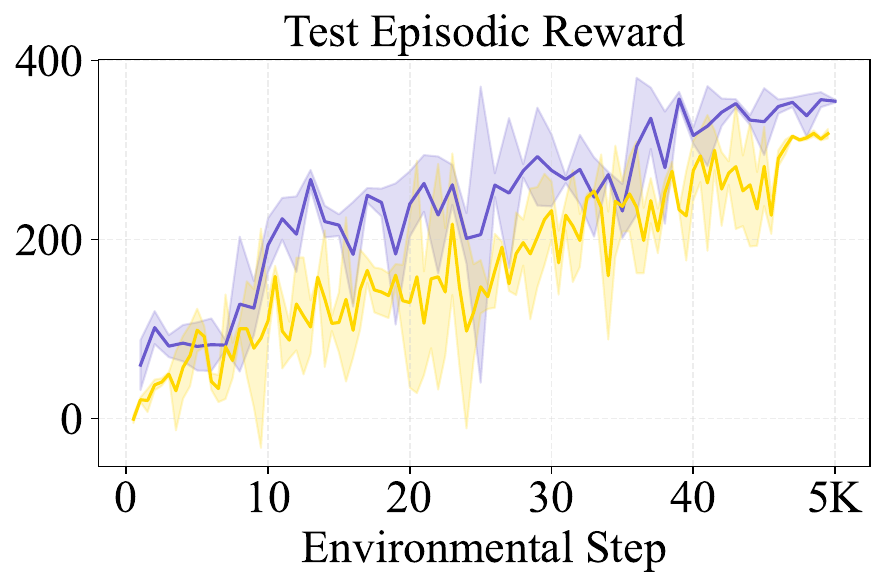}%
    \label{fig_4_4}} 
  \end{minipage}
  \\
  \subfloat{\includegraphics[width=6.43in]{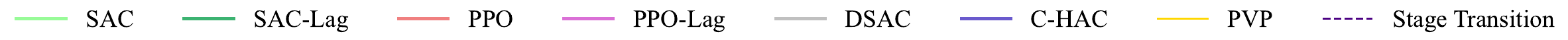}}
  \caption{Performance of different baselines.}
  \label{fig_6}
  \end{figure*}

\subsection{Baseline Methods}
The following baseline methods are used for comparison:
\begin{itemize}
  \item \textbf{RL}: Standard RL approaches including PPO, SAC, and DSAC; Safe RL approaches including PPO-Lag\cite{PPO_LAG} and SAC-Lag\cite{SAC_LAG}.
  \item \textbf{Offline RL and IL}: Offline RL approach using conservative Q-learning (CQL)\cite{CQL}; IL approaches using behavior cloning (BC) and generative adversarial imitation learning (GAIL) from human demonstrations.
  \item \textbf{HAC}: HAC approaches including PVP, D-PVP, HG-DAgger and IWR.
\end{itemize}
\indent \indent These baseline methods are implemented using RLLib. The training of baseline methods is conducted through five concurrent trails on Nvidia GeForce RTX 4080 GPUs. Each trail utilizes 2 CPUs with 6 parallel rollout workers. All baseline experiments about RL and IL are repeated five times using different random seeds. The experiments for C-HAC and PVP are conducted on a local computer and repeated three times.\\
\indent The evaluation metrics are divided into training and testing phases. During the training phase, we focus on data usage and total safety cost. The total safety cost represents the number of collisions during training, reflecting the potential dangers. In the testing phase, the primary metrics are episodic return, episodic safety cost, and success rate. The episodic safety cost is the average number of crashes in one episode. And the success rate is the ratio of episodes in which agents reach the destination to the total test episodes. For HAC methods, the total number of human data usage and the overall intervention rate are also reported. The overall intervention rate is the ratio of human data usage to total data usage, indicating the effort required from humans to teach the agents.
\subsection{Simulation Results}
\indent The performance of different baselines are listed in Table \ref{table1}. And Fig. \ref{fig_6} gives the learning curves.

 \begin{table*}[]
  \centering
  \caption{The performance of different baselines in the MetaDrive simulator.}
  \label{table1}
\begin{tabular}{lccccccc}
  \toprule
  \textbf{Method} & \multicolumn{3}{c}{\textbf{Training}} & \multicolumn{3}{c}{\textbf{Testing}} \\
  \cmidrule(lr){2-4} \cmidrule(lr){5-7}
  & \textbf{Human Data} & \textbf{Total Data} & \textbf{Safety Cost} & \textbf{Episodic Return} & \textbf{Episodic Safety Cost} & \textbf{Success Rate} \\
  \midrule
  SAC\textsuperscript{\cite{SAC}} & - & 1M & 8.86K $\pm$ 4.91K & 359.18 $\pm$ 18.51 & 1.00 $\pm$ 0.30 & 0.74 $\pm$ 0.17 \\
  PPO\textsuperscript{\cite{PPO}} & - & 1M & 45.12K $\pm$ 21.11K & 278.65 $\pm$ 35.07 & 3.92 $\pm$ 1.91 & 0.44 $\pm$ 0.14 \\
  SAC-Lag\textsuperscript{\cite{SAC_LAG}} & - & 1M & 8.73K $\pm$ 4.14K & 346.05 $\pm$ 20.57 & 0.64 $\pm$ 0.13 & 0.62 $\pm$ 0.17 \\
  PPO-Lag\textsuperscript{\cite{PPO_LAG}} & - & 1M & 26.55K $\pm$ 9.97K & 222.15 $\pm$ 49.66 & 0.88 $\pm$ 0.23 & 0.26 $\pm$ 0.07 \\
  DSAC\textsuperscript{\cite{DSAC}} & - & 1M & 7.44K $\pm$ 3.59K & 349.35 $\pm$ 22.15 & 0.47 $\pm$ 0.08 & 0.77 $\pm$ 0.09 \\
  \midrule
  Human Demo. & 50K & - & 23 & 377.523 & 0.39 & 0.97 \\
  \midrule
  CQL\textsuperscript{\cite{CQL}} & 50K (1.0) & - & - & 93.12 $\pm$ 16.31 & 1.45 $\pm$ 0.15 & 0.09 $\pm$ 0.05 \\
  BC\textsuperscript{\cite{BC}} & 50K (1.0) & - & - & 59.13 $\pm$ 8.92 & 0.12 $\pm$ 0.03 & 0 $\pm$ 0 \\
  GAIL\textsuperscript{\cite{gail}} & 50K (1.0) & - & - & 34.78 $\pm$ 3.92 & 1.07 $\pm$ 0.13 & 0 $\pm$ 0 \\
  \midrule
  HG-Dagger\textsuperscript{\cite{DAGGEROther02}} & 34.9K (0.70) & 0.05M & 56.13 & 142.35  & 2.1 & 0.30 \\
  IWR\textsuperscript{\cite{IWR}} & 37.1K (0.74) & 0.05M & 48.78 & 329.97  & 4.00 & 0.70 \\
  \midrule
  PVP\textsuperscript{\cite{PVP}} & 15K & 0.05M & 35.67 $\pm$ 4.32 & 338.28 $\pm$ 9.72 & 0.898 $\pm$ 0.15 & 0.81 $\pm$ 0.04 \\
  D-PVP & 15K  & 0.05M & 32.12$\pm$4.68 & 353.39$\pm$12.34 & 0.31$\pm$0.03 & 0.83$\pm$0.05 \\
  \textbf{C-HAC} & \textbf{15K}  & \textbf{0.05M+0.95M} & \textbf{176.00 $\pm$ 31.82}  & \textbf{392.92$\pm$15.70} &  \textbf{0.16$\pm$0.02}  & \textbf{0.91$\pm$0.02} \\
  \bottomrule
  \end{tabular}
\end{table*}
\textbf{Comparison with RL approaches.} Fig. \ref{fig_6}\subref{fig_6_1} highlights the training and testing performance of the proposed C-HAC compared to standard RL and safe RL algorithms. In the MetaDrive environment, C-HAC achieves an average return of 392.92 with a safety cost 0.16, significantly outperforming SAC, PPO, DSAC, SAC-Lag, and PPO-Lag. It also maintains higher success rates across all scenarios. During the learning-from-demonstration stage, C-HAC realizes a return of 353.39 and an 83\% success rate within 50,000 steps, completing training in approximately one hour. In the subsequent RL enhancement stage, the agent further improves its return, demonstrating the benefits of reward-guided policy refinement. Moreover, C-HAC ensures safety, recording an average of 35.67 safety violations in the demonstration stage and 140.33 in the enhancement stage, while maintaining rapid convergence.\\
\indent \textbf{Comparison with Offline RL and IL methods.} A dataset of 100 episodes (50K steps) of human driving data was collected for training offline RL and IL baselines. The dataset achieved a 97\% success rate, an average return of 377.52, and a safety cost of 0.39. Using this data, CQL, BC, and GAIL were trained. As presented in Table \ref{table1}, C-HAC significantly outperforms these methods, with test success rates for the baselines remaining below 10\%. Notably, GAIL achieves a 0\% success rate due to its reliance on strictly matching the expert data distribution, leading to poor generalization in unseen scenarios. Additionally, the episodic returns of BC, CQL, and GAIL are far below those achieved by C-HAC. Unlike IL methods, which optimize the agent to mimic expert actions at each timestep, C-HAC employs trajectory-based learning, promoting actions that maximize future rewards rather than simple imitation. Furthermore, C-HAC collects expert data online, effectively mitigating the distribution shift problem in offline RL and achieving superior performance.\\
\begin{figure}[]
  \centering
  \subfloat{\includegraphics[width=1.65in]{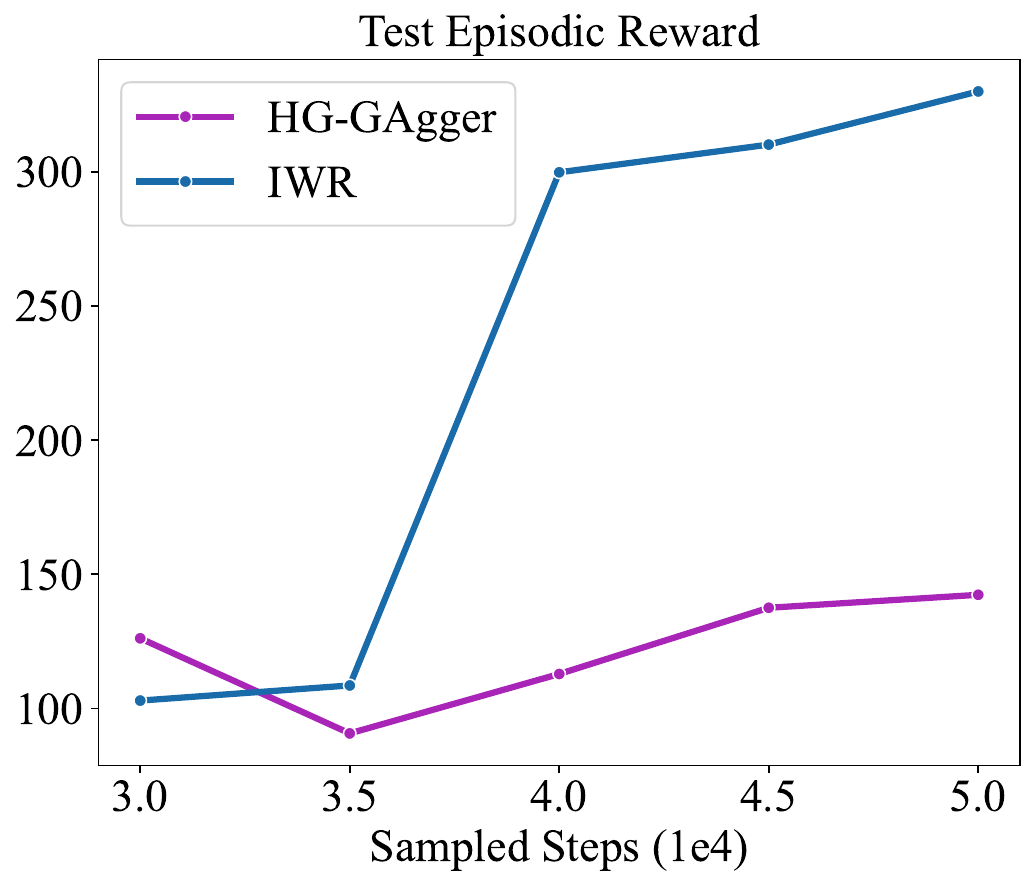}}%
  \hfil
  \subfloat{\includegraphics[width=1.65in]{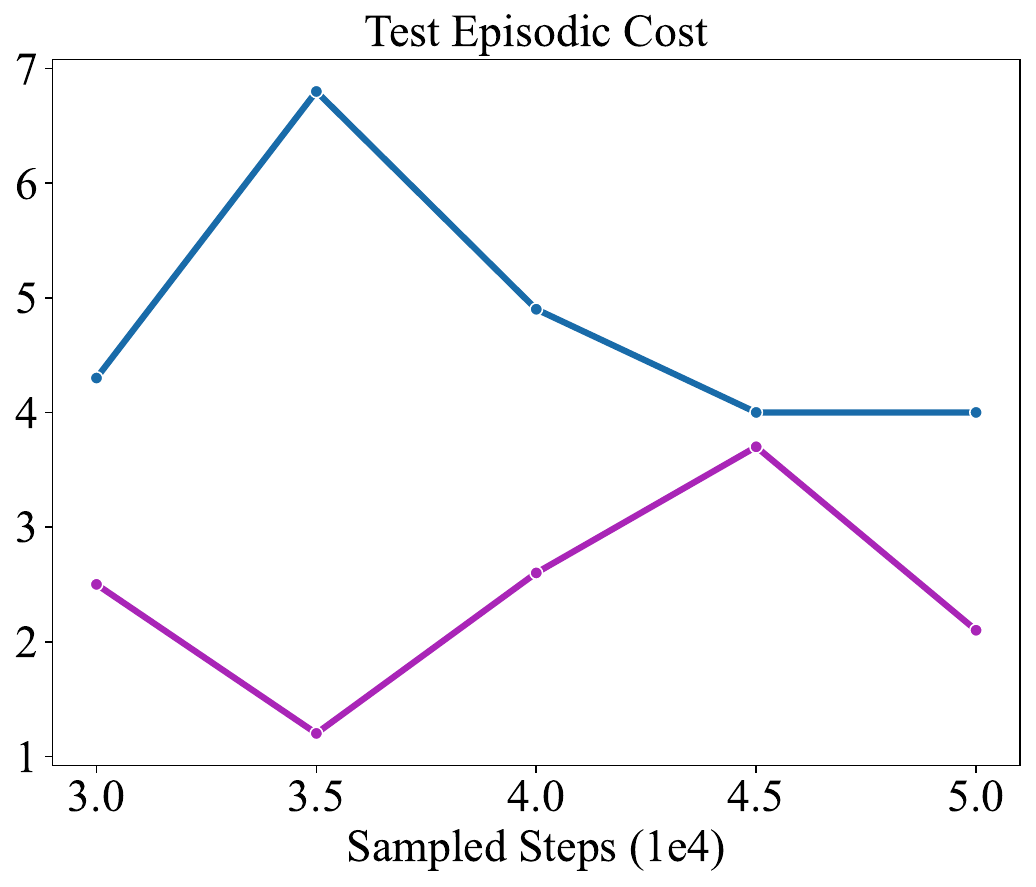}}%
  \caption{Performance comparison of C-HAC with HAC methods.}
  \label{fig_7}
  \end{figure}
\indent \textbf{Comparison with HAC Baseline.} During the warm-up phase, both HG-DAgger and IWR utilize a dataset of 30K pre-collected samples, which are later added to their replay buffers during training. Both methods also undergo four rounds of training using BC, with at least 5K human-guided samples in each round. As shown in Fig. \ref{fig_7}, only IWR achieves a satisfactory success rate due to its emphasis on prioritizing human intervention data. This approach enables the agent to learn critical maneuvers and avoid compounding errors. In contrast, HG-DAgger struggles with limited demonstrations, as it lacks a re-weighting mechanism for human-guided samples. The proposed C-HAC consistently performs better than these baselines, demonstrating superior performance during both the initial guidance phase and final evaluation. These results highlight the method's ability to effectively leverage limited human input for robust policy learning. The C-HAC was further evaluated against the PVP approach. As shown in Fig. \ref{fig_6} and Table \ref{table1}, C-HAC exhibits faster convergence during the human-guided phase compared to PVP. It achieves a more rapid reduction in the average intervention probability and obtains a higher final test return while utilizing a comparable amount of human guidance data. Furthermore, during the subsequent RL continuous enhancement stage, C-HAC demonstrates superior performance, achieving higher returns than PVP.\\
\begin{figure}[]
  \centering
  \subfloat{\includegraphics[width=1.65in]{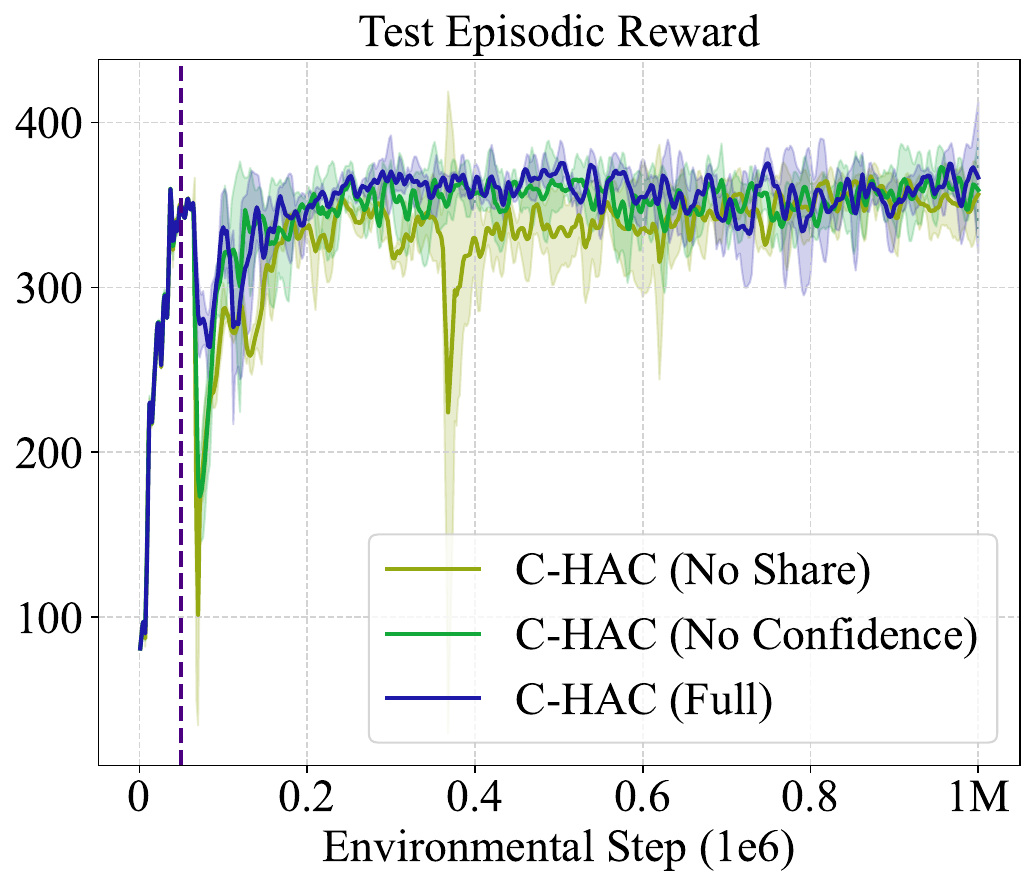}}%
  \hfil
  \subfloat{\includegraphics[width=1.65in]{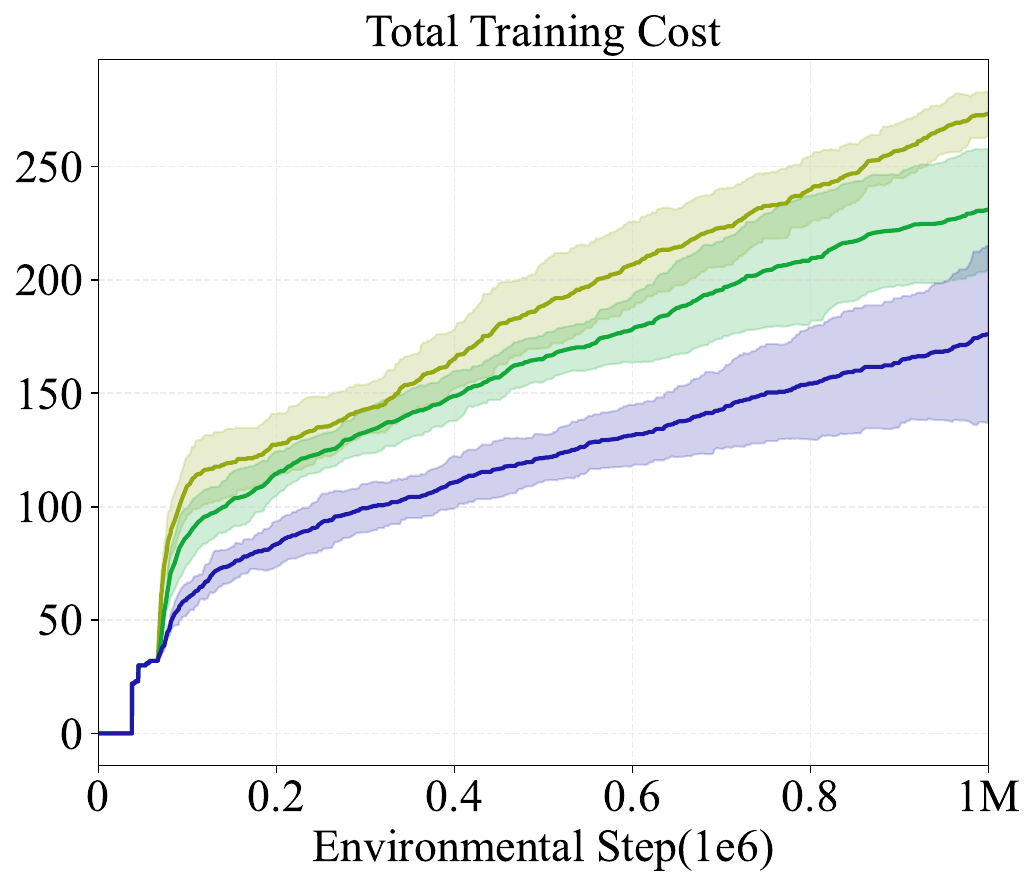}}%
  \caption{Impact of shared control mechanism and policy confidence evaluation on policy performance}
  \label{fig_8}
  \end{figure}
\indent To evaluate the effectiveness of the shared control mechanism and the policy confidence evaluation algorithm, we conducted the following comparative experiments. First, a baseline method (denoted as C-HAC (No Share)) excludes the shared control mechanism and directly trains the policy to maximize cumulative rewards after the learning-from-demonstration stage. Second, another baseline (denoted as C-HAC (No Confidence)) excludes the policy confidence evaluation algorithm, instead comparing the mean expected returns of policies without considering the confidence factor in formula (\ref{confidence intervention}); this method is equivalent to the case of setting $\delta=0.5$. Finally, our proposed approach (denoted as C-HAC (Full)) incorporates both the shared control mechanism and the policy confidence evaluation algorithm, with $\delta$ being 0.15. Fig. \ref{fig_8} depicts the Test Episodic Reward and Total Training Cost for these three methods. Method C-HAC (No Share) shows a 70\% decline in policy performance during the RL continuous enhancement stage, with significantly higher training costs. Method C-HAC (No Confidence), which employs the shared control mechanism, reduces the decline to 50\% and incurs a training cost of 200. Method C-HAC (Full), using the policy confidence evaluation algorithm, achieves the smallest decline of 20\% and the lowest training cost. These results demonstrate that the shared control mechanism and policy confidence evaluation algorithm significantly enhance the stability and safety of policy learning.
\subsection{Visualization}
In Fig. \ref{fig_9}, the action sequences of agents trained with C-HAC and PVP are visualized. The angle and length of each arrow represent the steering angle and acceleration, respectively, with the human subject’s actions highlighted in yellow. Compared to PVP, the action sequences generated by C-HAC are notably smoother and align more closely with those of human drivers.\\
\begin{figure}[]
  \centering
  \subfloat{\includegraphics[width=1.65in]{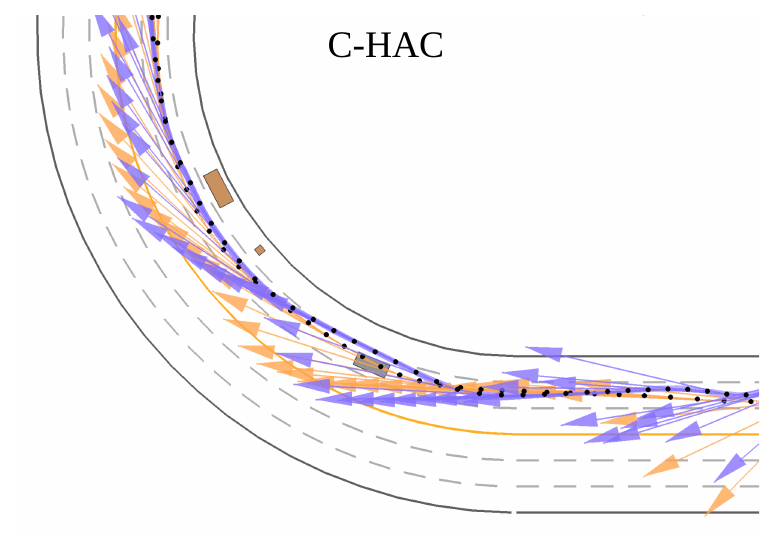}}%
  \hfil
  \subfloat{\includegraphics[width=1.65in]{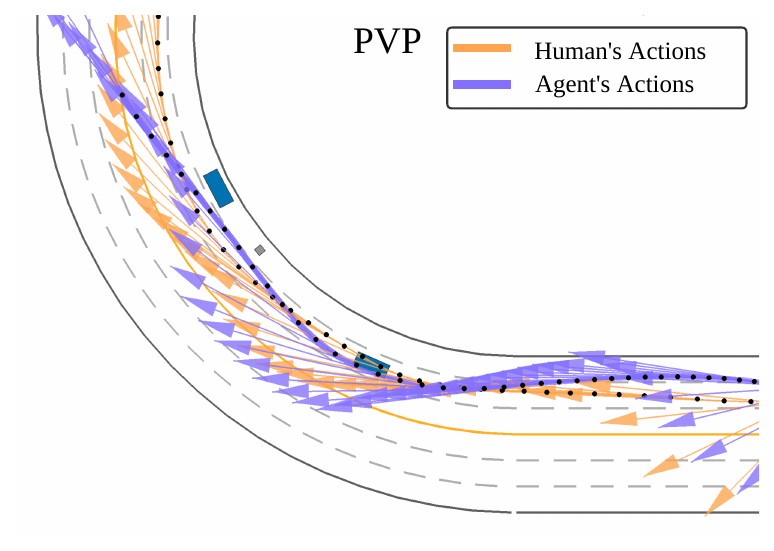}}%
  \caption{The action sequences generated by C-HAC and PVP agents in the same MetaDrive map.}
  \label{fig_9}
\end{figure}
\begin{figure}[]
  \centering
  \includegraphics[width=3.5in]{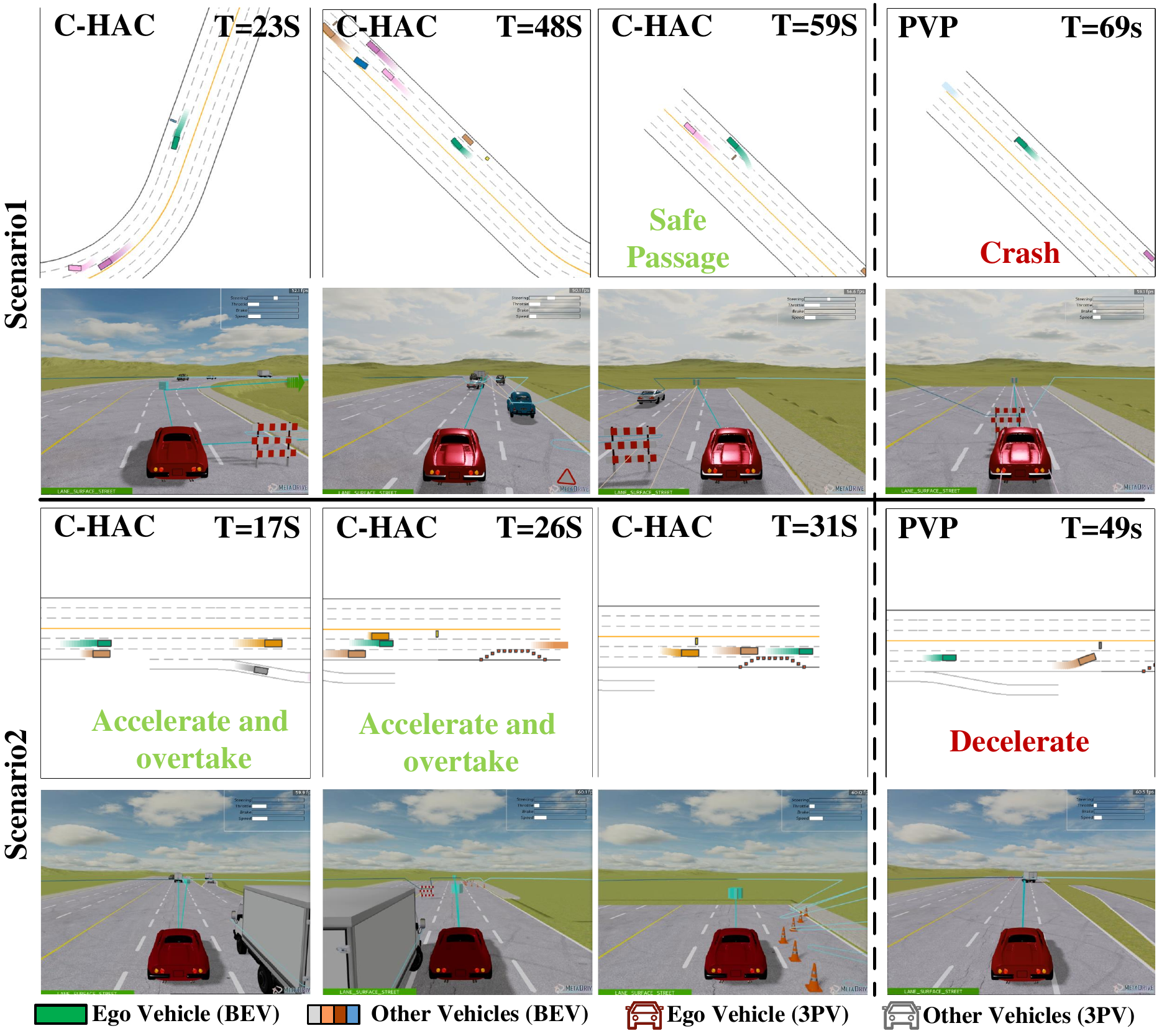}
  \caption{Comparative visualizations of D-PVP and C-HAC across two driving scenarios}
  \label{vis_2}
  \end{figure}
\indent In Fig. \ref{vis_2}, Scenario 1 compares the trajectories of an agent trained with D-PVP during the learn-from-demonstration stage to those refined by C-HAC in the RL continuous enhancement stage. When the D-PVP policy leads to collisions, C-HAC leverages its shared control mechanism to identify and correct these shortcomings. This capability explains C-HAC's higher success rate relative to D-PVP.\\
\indent In Fig. \ref{vis_2}, Scenario 2 depicts a congested environment. The D-PVP policy chooses to decelerate and wait for surrounding vehicles, whereas C-HAC accelerates at the right moment to overtake. Consequently, C-HAC completes the scenario more quickly and achieves higher episodic rewards.
\begin{figure}[h]
  \centering

  \subfloat{\includegraphics[width=2.5in]{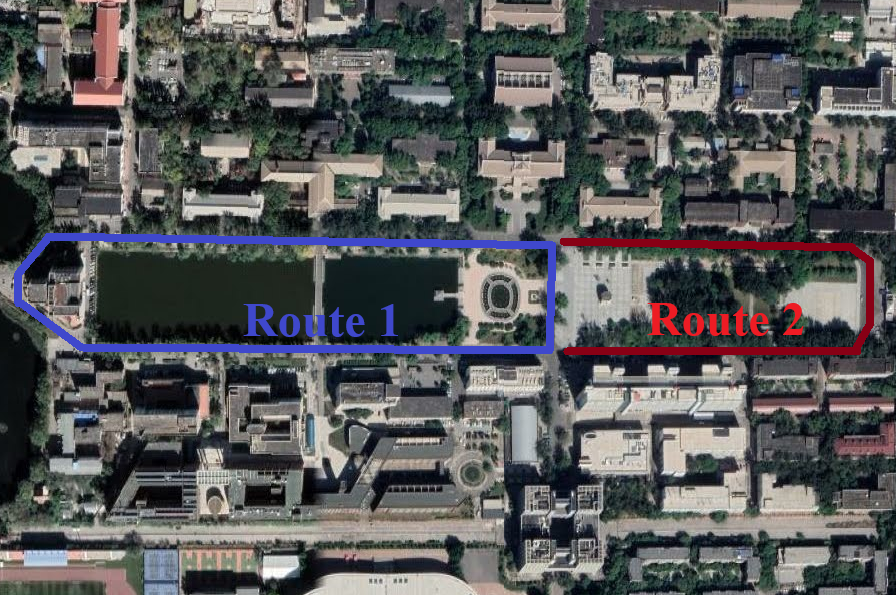}}%
  \caption{Routes for training and testing in real-world experiments}
  \label{real_v}
  \end{figure}
  \begin{figure}[h]
    \centering
    \subfloat{\includegraphics[width=3.4in]{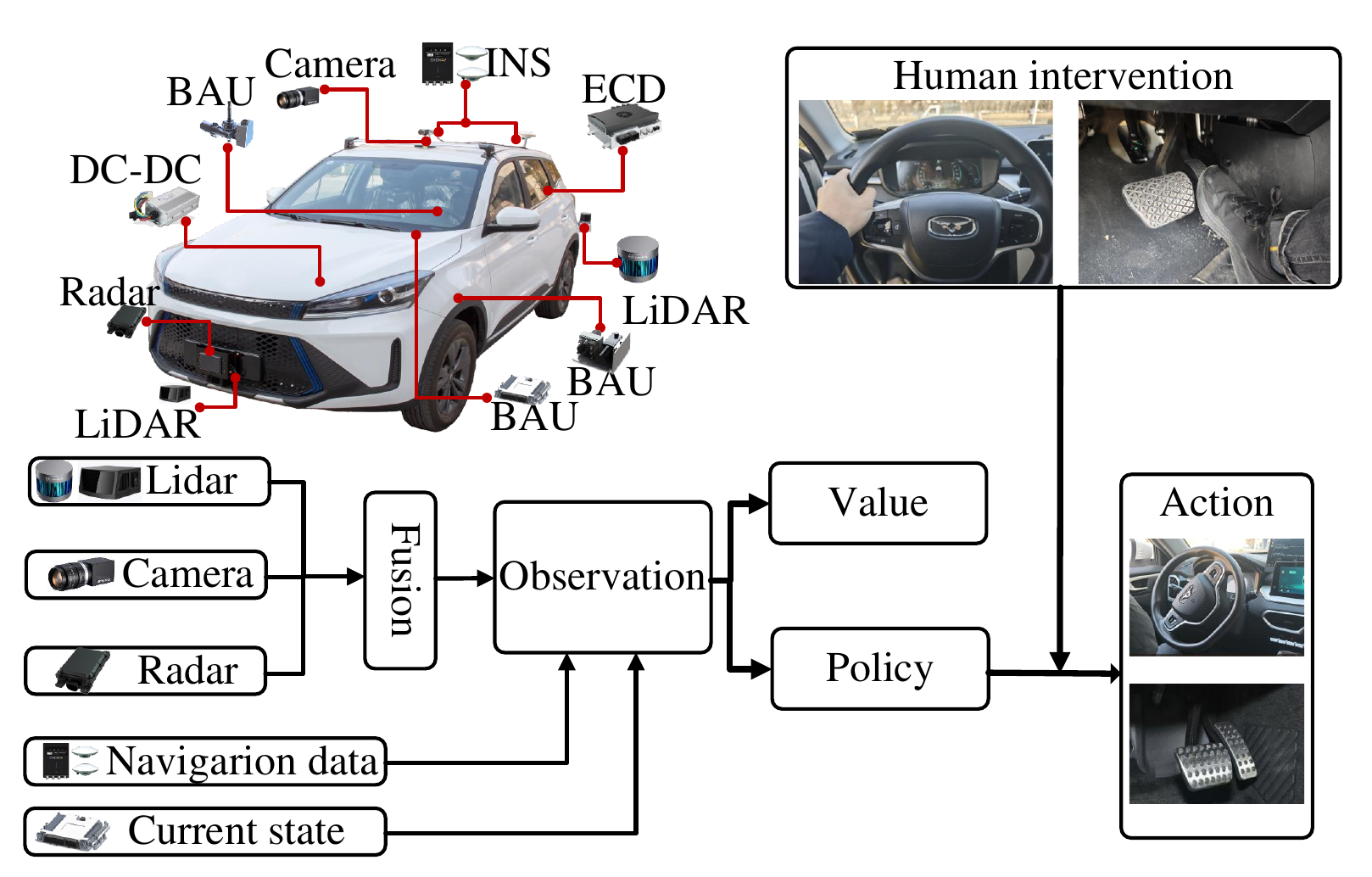}}%
    \caption{Architecture of UGV setup with human intervention.}
    \label{real_frame}
    \end{figure}
\subsection{Real-World Driving Experiment}
\begin{figure*}[]
  \centering
  \subfloat[UGV executes left turn while avoiding pedestrian.]{\includegraphics[width=2.2in]{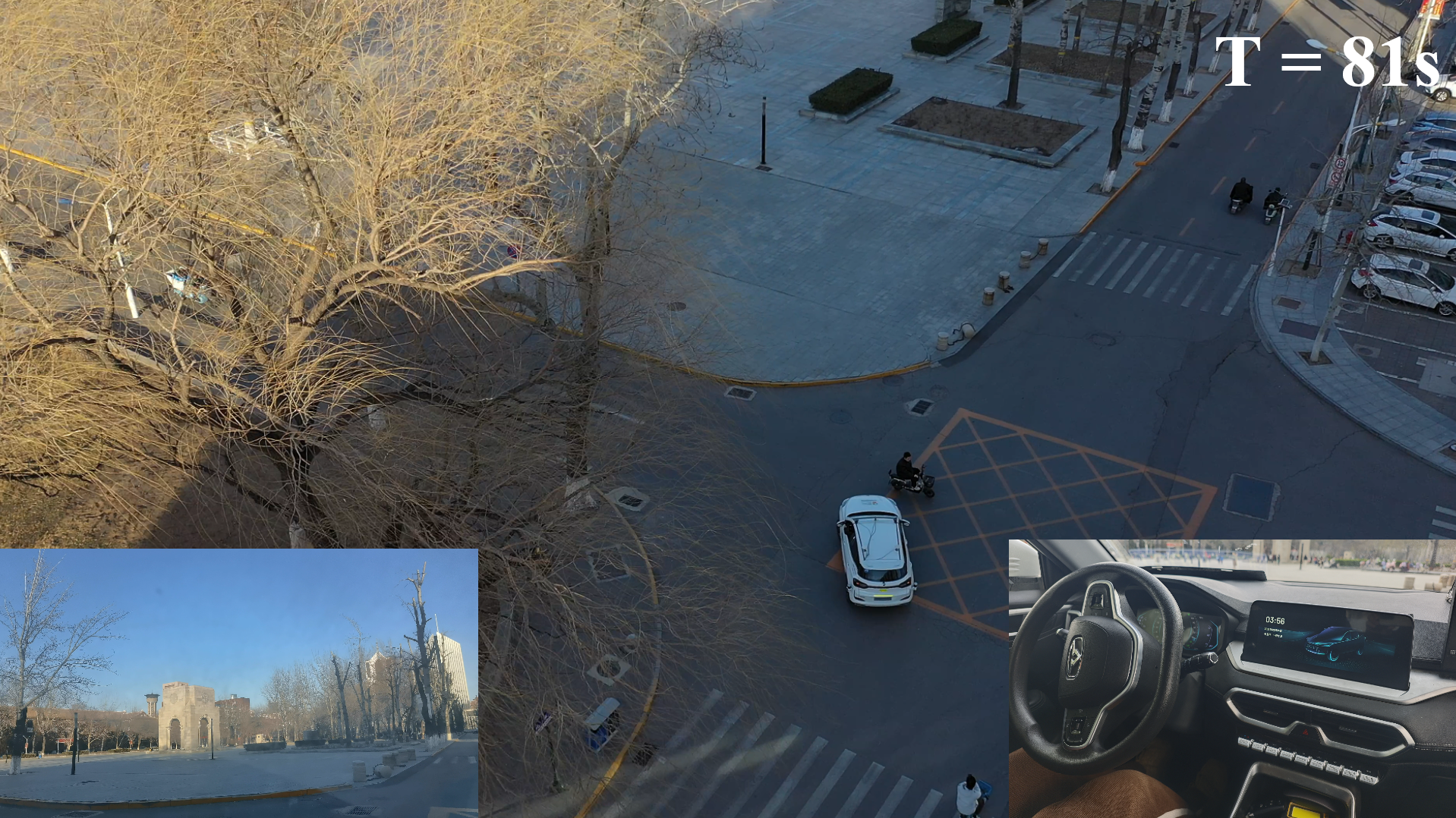}}%
  \hfil
  \subfloat[UGV slows down and stops to yield to a crossing pedestrian.]{\includegraphics[width=2.2in]{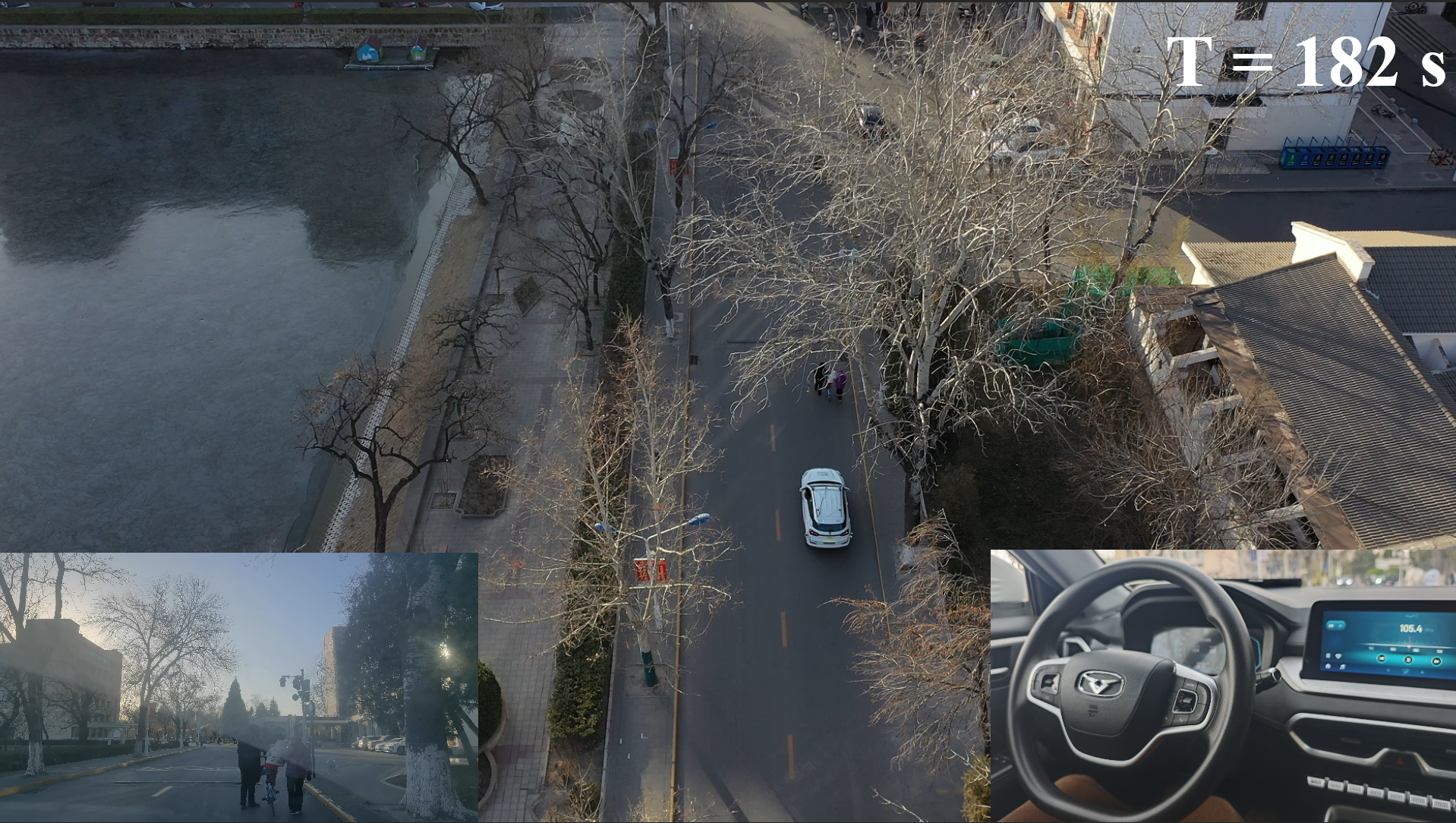}}%
  \hfil
  \subfloat[UGV executes a sharp turn.]{\includegraphics[width=2.2in]{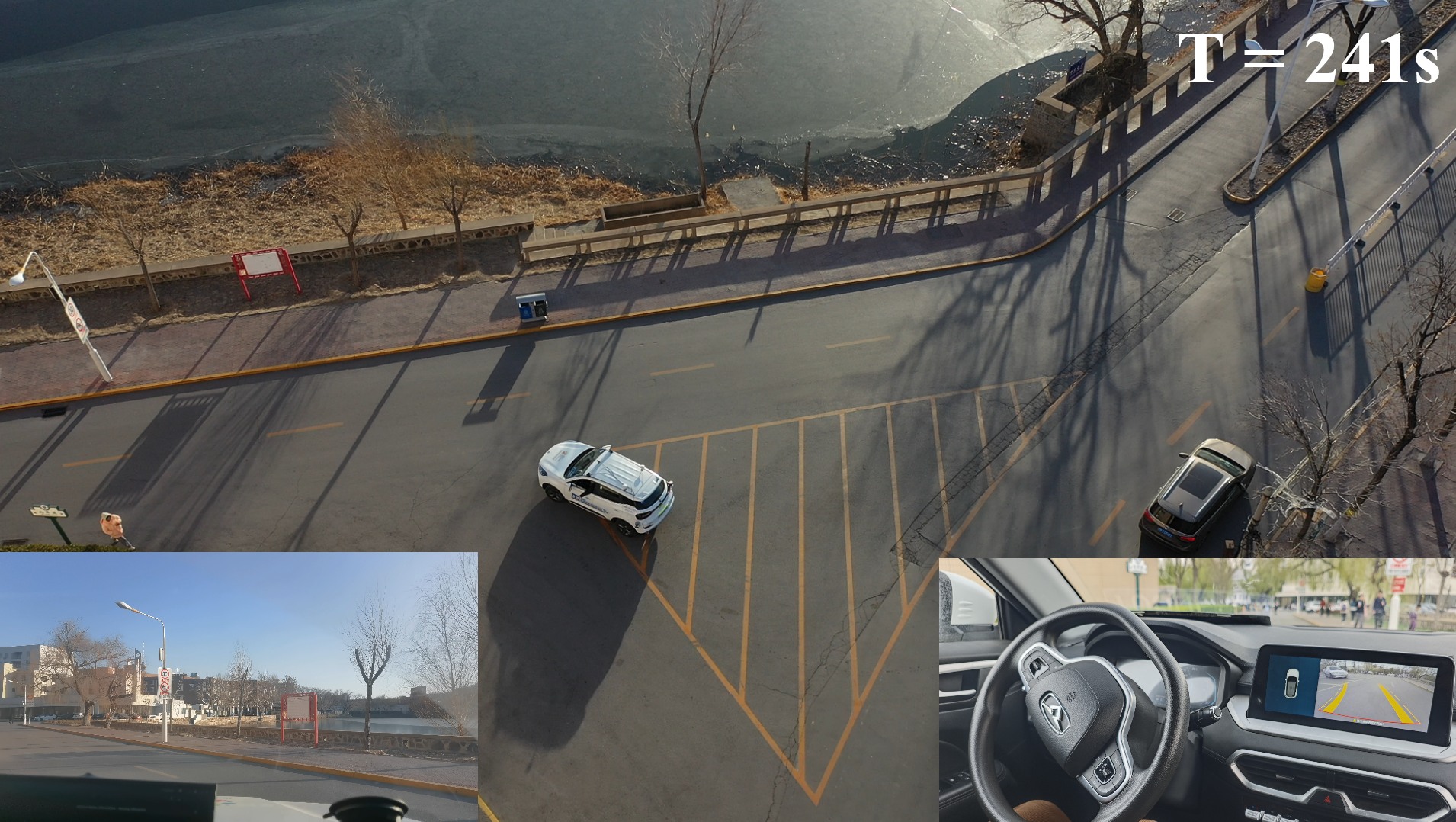}}%
  \\
  \subfloat[UGV maneuvers around an obstacle.]{\includegraphics[width=2.2in]{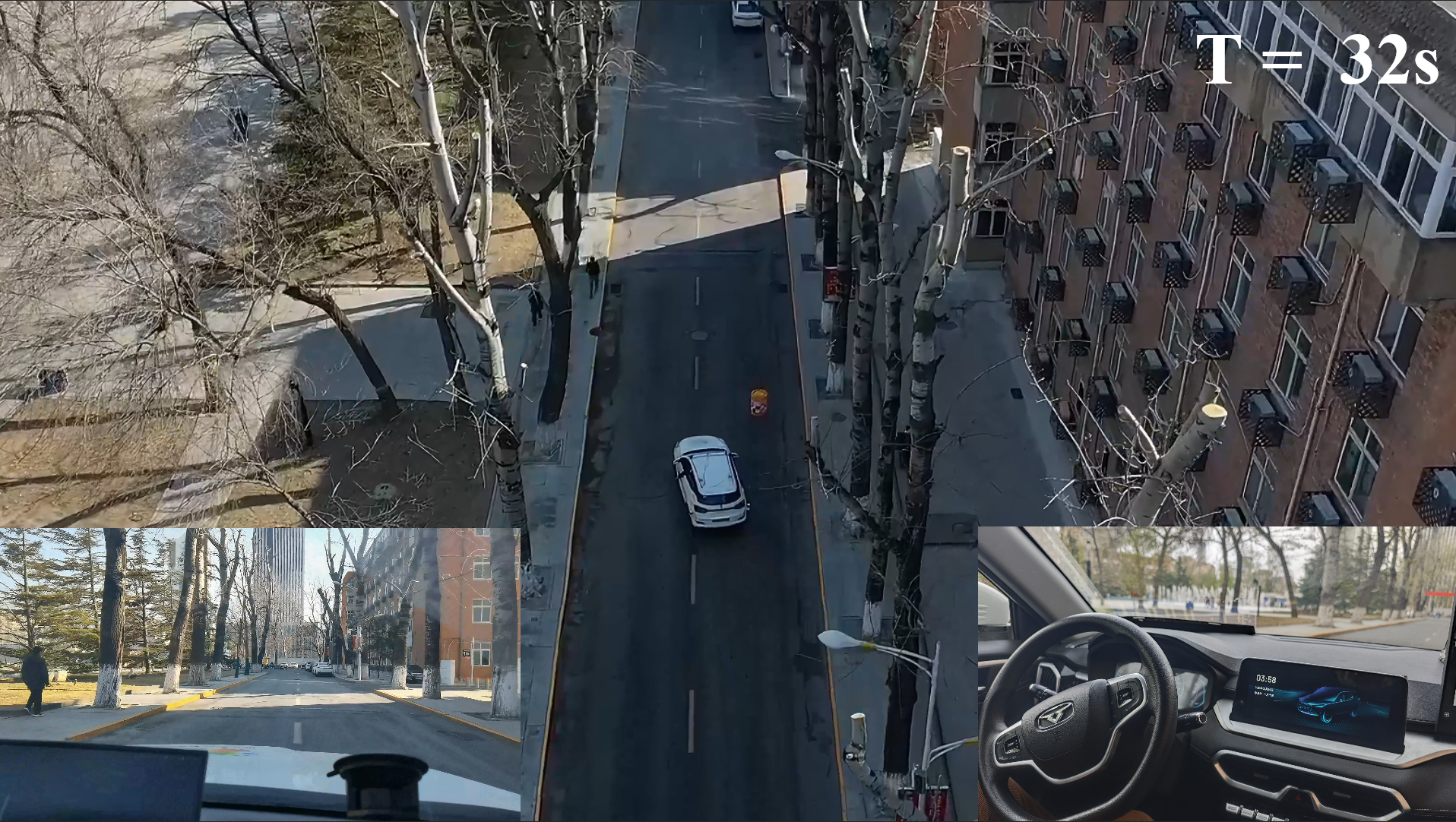}}%
  \hfil
  \subfloat[UGV navigates past a stationary vehicle.]{\includegraphics[width=2.2in]{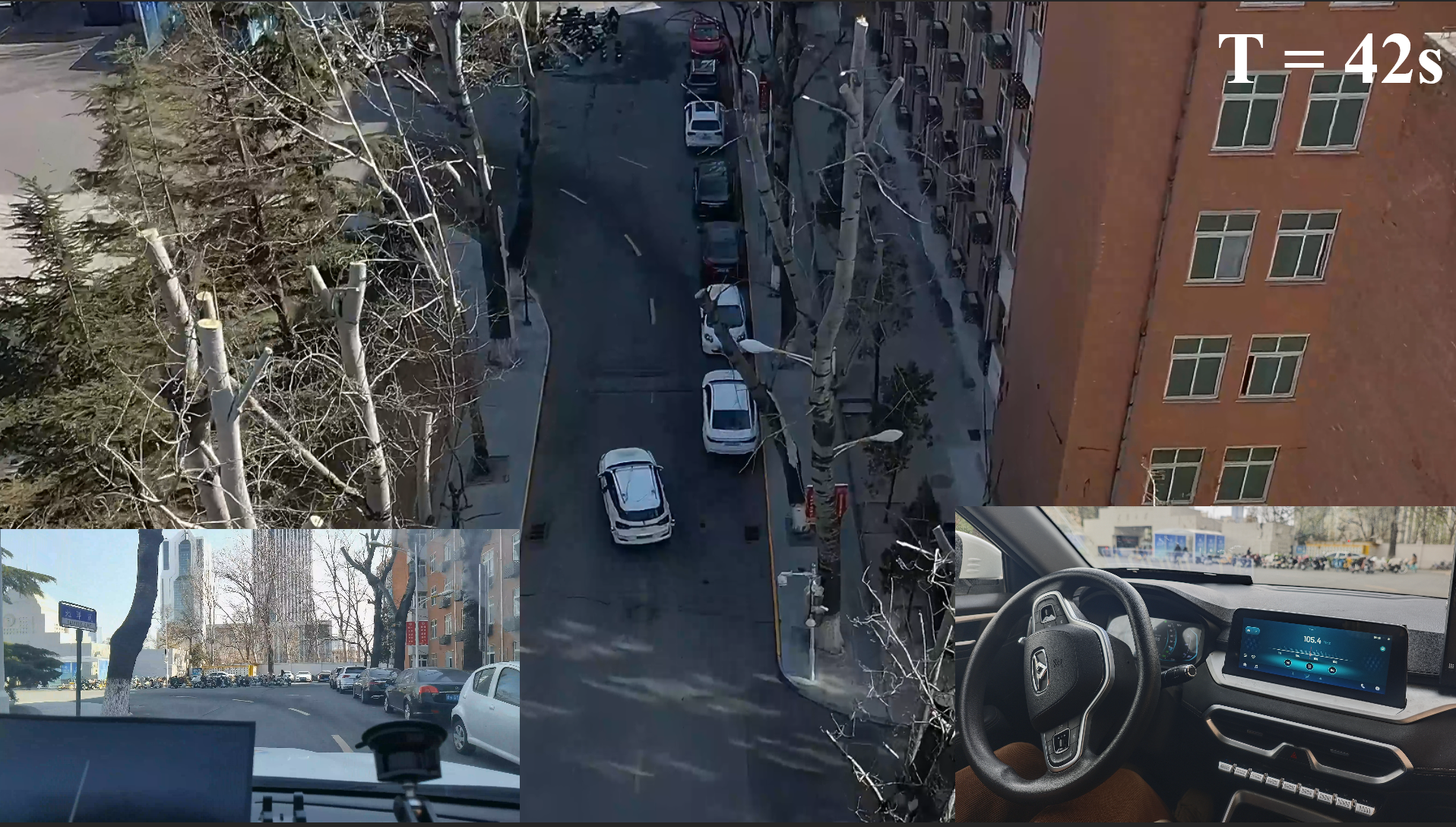}}%
  \hfil
  \subfloat[UGV manages an intersection with heavy traffic.]{\includegraphics[width=2.2in]{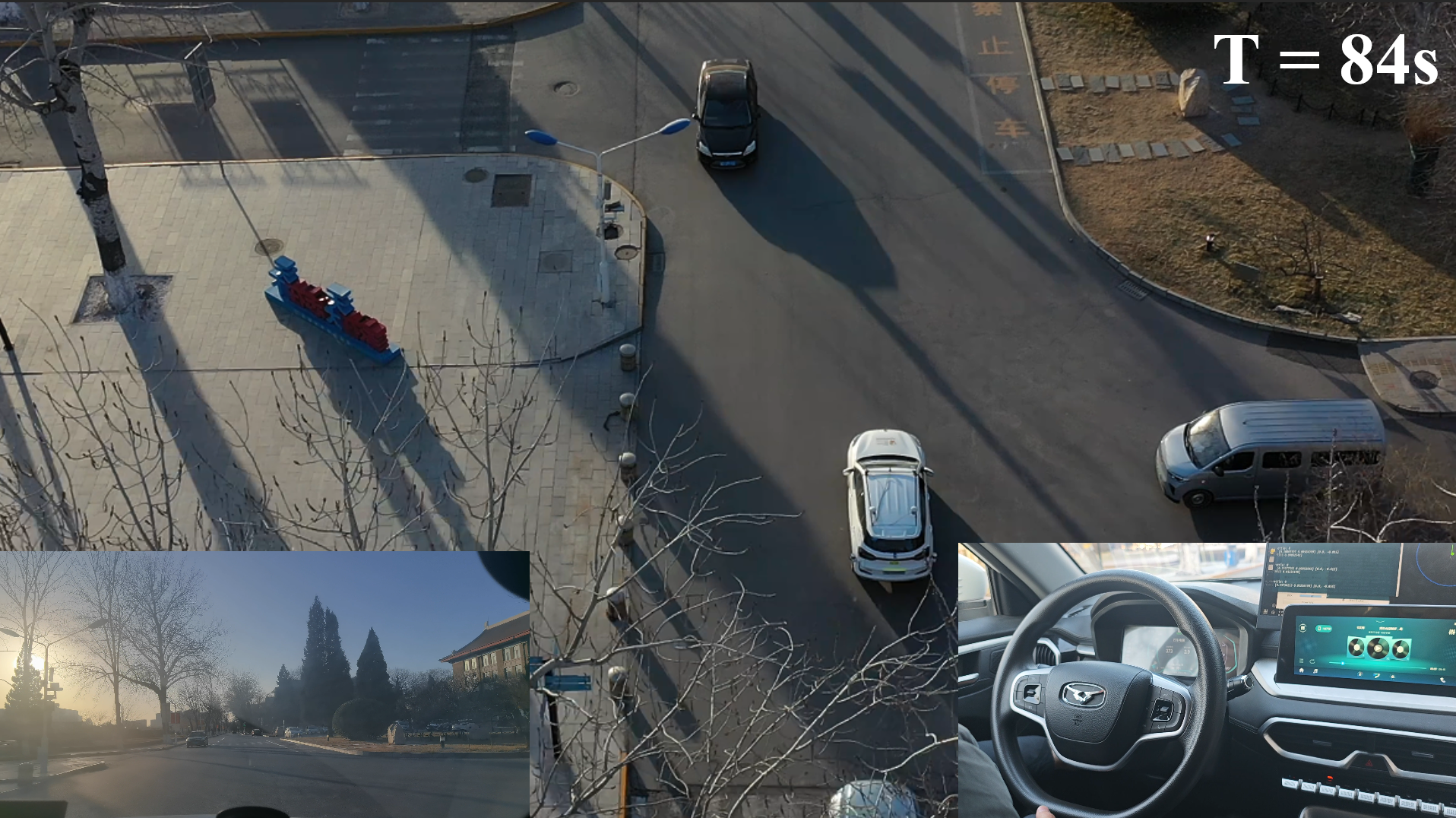}}%
  \caption{Real-world driving performance on Route 1 and Route 2.}
  \label{route}
  \end{figure*}

  \begin{figure}[]
    \centering
    \subfloat[Action outputs on Route 1.]{\includegraphics[width=1.7in]{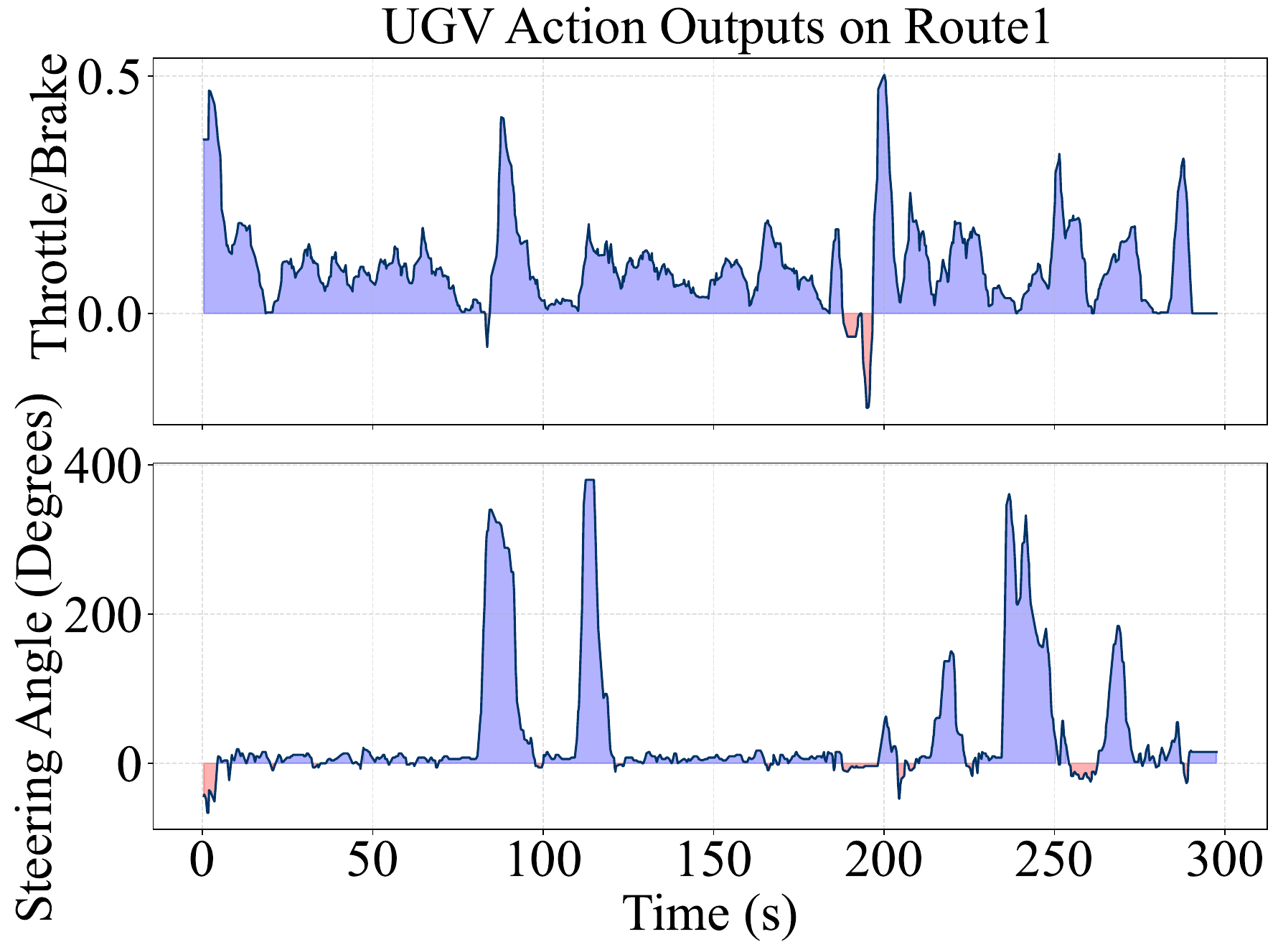}}%
    \hfil
    \subfloat[Action outputs on Route 2.]{\includegraphics[width=1.7in]{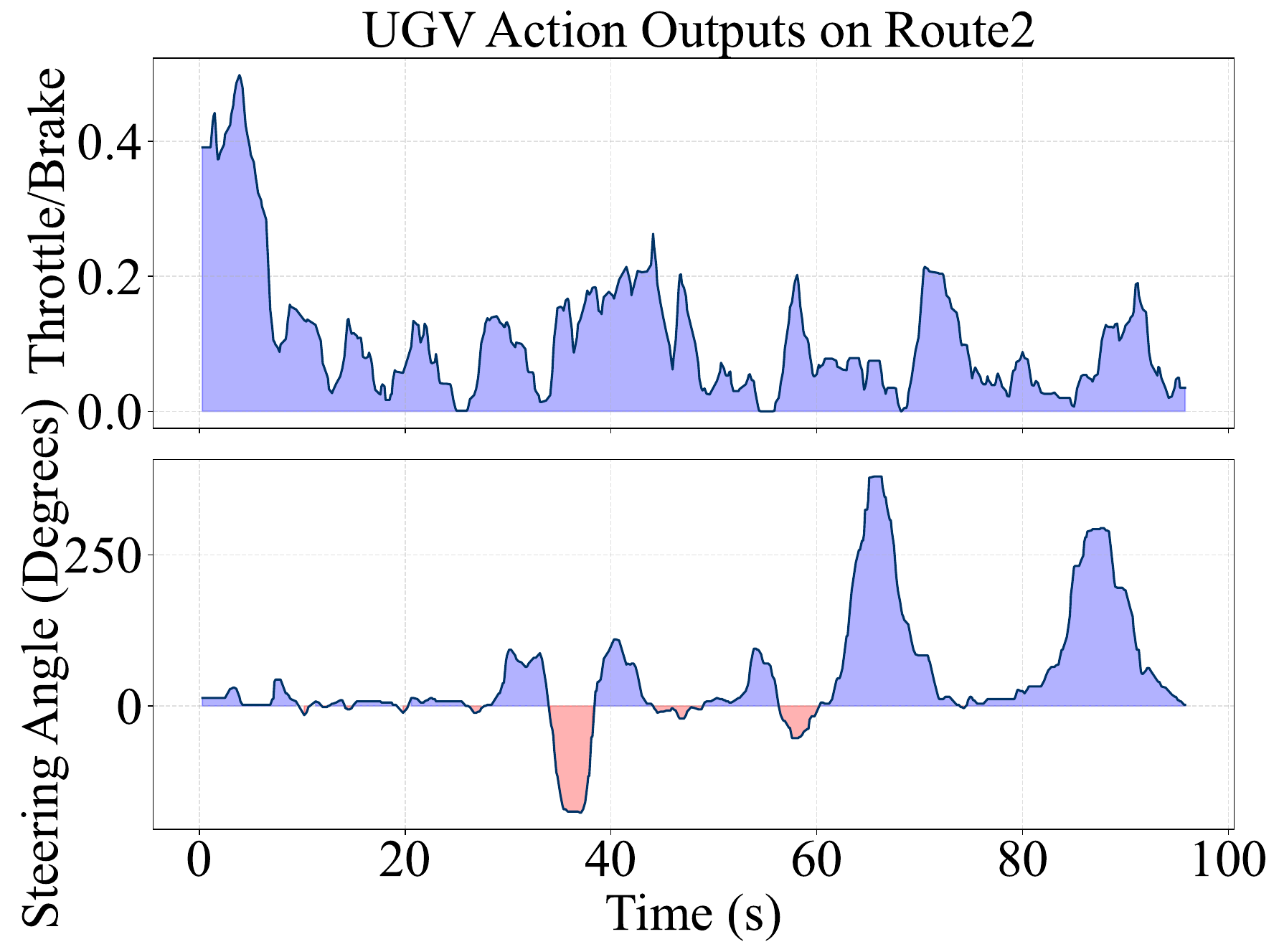}}%
    \\
    \subfloat[Speed on Route 1.]{\includegraphics[width=1.7in]{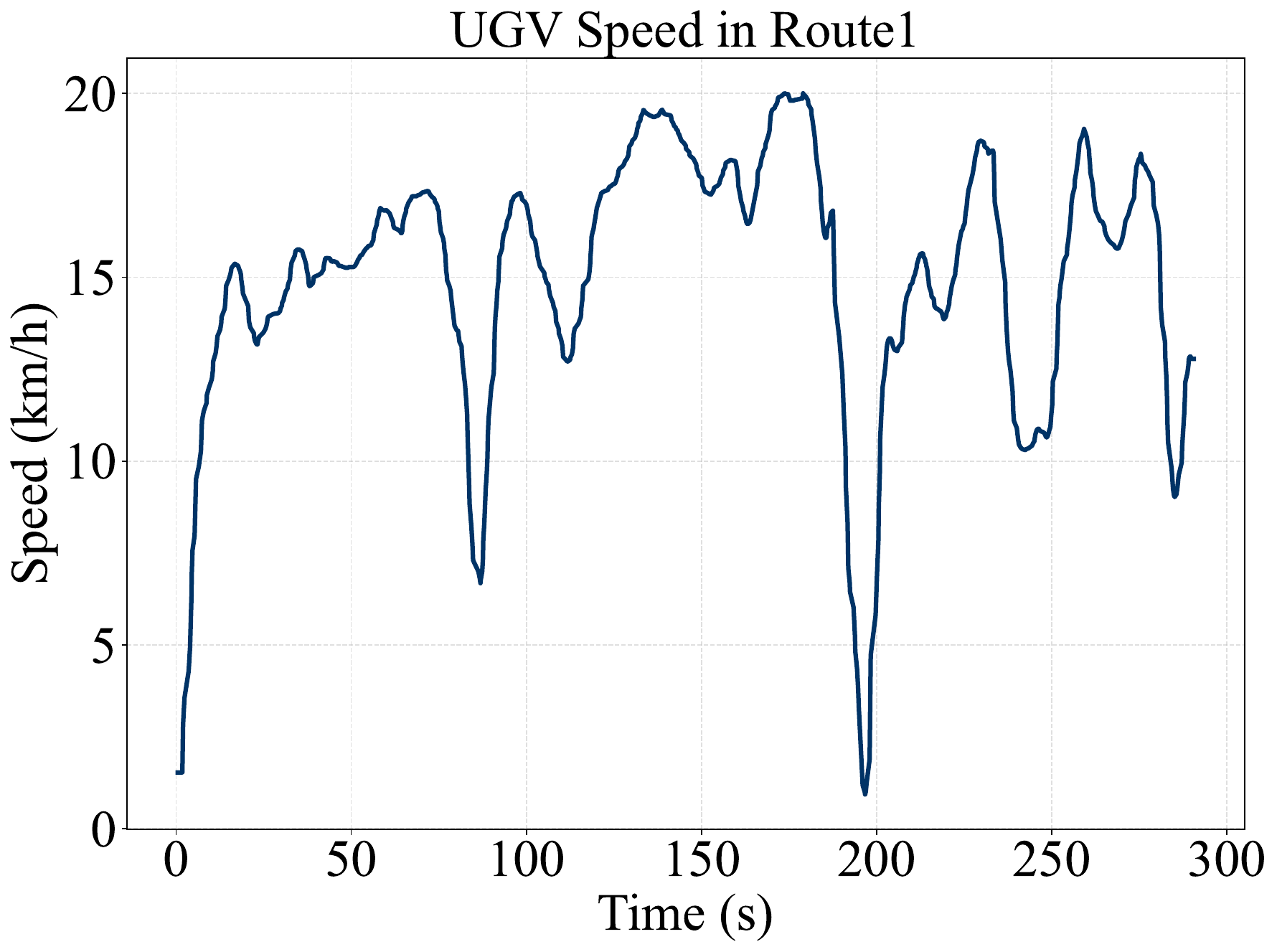}}%
    \hfil
    \subfloat[Speed on Route 2.]{\includegraphics[width=1.7in]{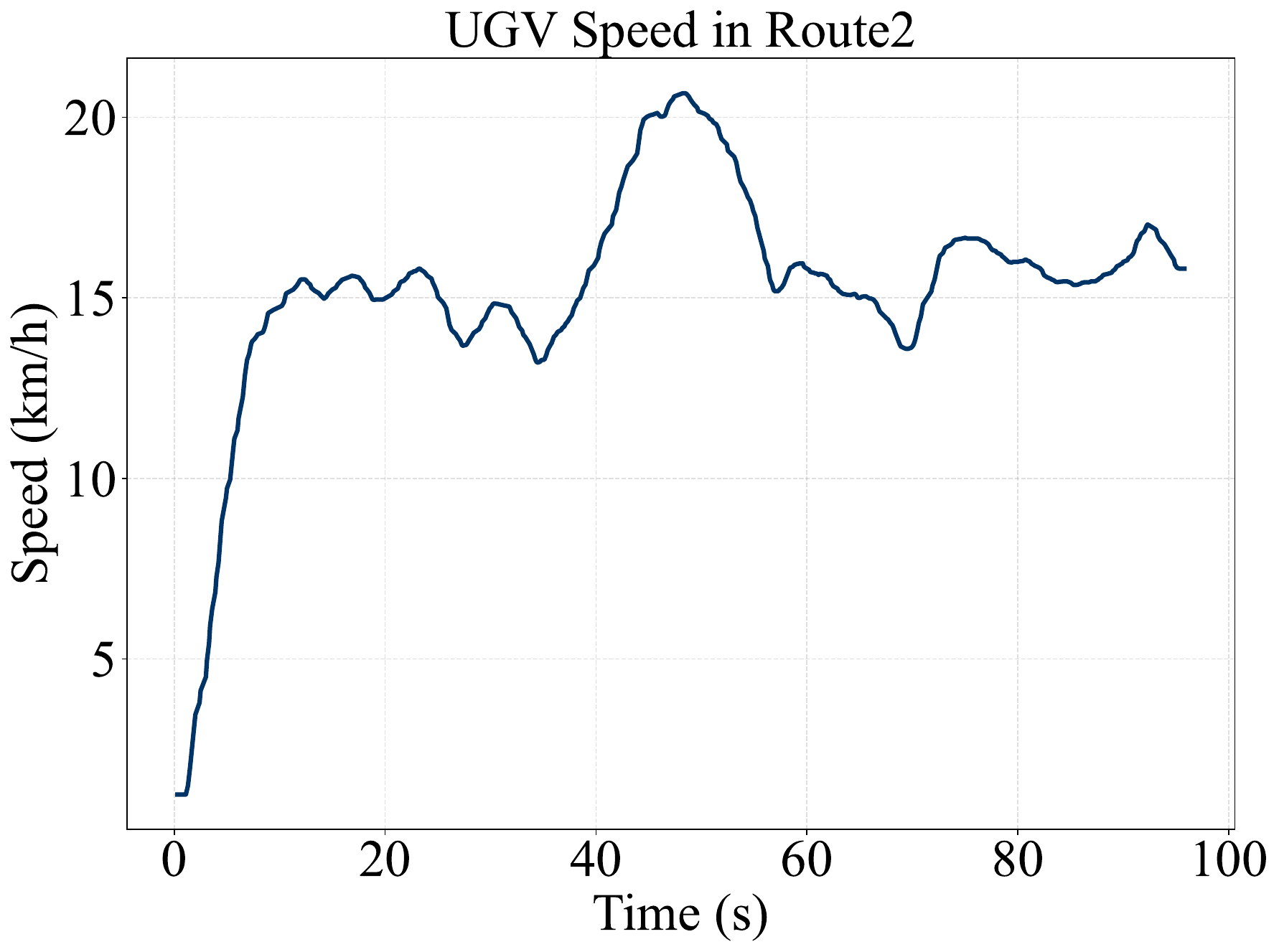}}%

    \caption{Action outputs and speeds of the vehicle in real-world scenarios}
    \label{action}
    \end{figure}
As shown in Fig. \ref{real_v}, the real-world training process takes place on the campus roads of Tianjin University. Each route consists of multiple checkpoints that specify both position and driving commands. Route 1 is used for training, while Route 2 is designated for generalization testing. As illustrated in Fig. \ref{real_frame}, the UGV localizes via an integrated navigation system (INS) and perceives its surroundings using LiDAR, camera, and radar. An Nvidia Jetson AGX Orin is responsible for the perception tasks. Training is performed on a separate GPU, and the resulting control commands are sent to the base adapter unit (BAU). Random pedestrians, bicycles, and vehicles introduce natural variability, making the training conditions more challenging. The definitions of the observation space and action space are given below:\\
  \indent \textbf{Observation space}: This is a continuous space comprising (a) the UGV's current state, including speed, lateral offset from the lane center, and heading angle relative to the lane center; (b) surrounding information, which is derived by fusing detection results from LiDAR, camera, and radar, then converting these into a 240-dimensional LiDAR-like distance vector for nearby vehicles and obstacles—following a representation approach similar to MetaDrive; and (c) navigation data, consisting of the next 30 checkpoints and driving instructions such as go straight, turn left or turn right.\\
  \indent \textbf{Action space}: The action space is continuous, with two components: acceleration and steering angle.\\
  \indent During training, the UGV navigates checkpoints on Route 1 while actively avoiding obstacles and other vehicles. A human operator may intervene at any time by pressing the autopilot mode switch button or using the steering wheel and throttle/brake pedals to manually override the system. The UGV continues driving under supervision until reaching a predefined total of 100,000 steps, at a policy execution frequency of 10 Hz, resulting in roughly two hours of training.\\
  \indent In testing on Route 1, the UGV successfully completes the entire route. As shown in Fig. \ref{route}(a), it executes a left turn while avoiding a pedestrian, and in Fig. \ref{route}(b), it slows down and stops to yield to a crossing pedestrian. Additionally, as depicted in Fig. \ref{route}(c), the UGV executes a sharp turn, with the corresponding actions and speed presented in Fig. \ref{action}(a) and Fig. \ref{action}(c). For generalization testing, the UGV is evaluated on Route 2, with action and speed output shown in Fig. \ref{action}(b) and Fig. \ref{action}(d). In Fig. \ref{route}(f), it maneuvers around an obstacle, while Fig. \ref{route}(g) shows the UGV navigating past a stationary vehicle. Finally, as illustrated in Fig. \ref{action}(b), it effectively manages an intersection with heavy traffic.

\section{Conclusion}
This paper proposes a confidence-guided human-AI collaboration method that enhances autonomous driving policy learning through the integration of distributional proxy value propagation and distributional soft actor-critic. Such a treatment overcomes the limitations of purely human-guided or purely self-learning strategies by providing a structured two-stage procedure comprising learn from demonstration and RL continuous enhancement. Experimental results on the MetaDrive benchmark demonstrate substantial improvements in both safety and overall performance compared to conventional RL, safe RL, IL, and other HAC methods. By incorporating a shared control mechanism and a policy confidence evaluation algorithm, this method efficiently reduces human supervision while preserving essential human-guided behaviors. These findings underscore the potential of unifying human expertise and autonomous exploration within a single framework, thereby offering a more scalable and robust pathway toward real-world autonomous driving systems.

\bibliographystyle{IEEEtran}
\bibliography{references}

\begin{thebibliography}{10}
\providecommand{\url}[1]{#1}
\csname url@samestyle\endcsname
\providecommand{\newblock}{\relax}
\providecommand{\bibinfo}[2]{#2}
\providecommand{\BIBentrySTDinterwordspacing}{\spaceskip=0pt\relax}
\providecommand{\BIBentryALTinterwordstretchfactor}{4}
\providecommand{\BIBentryALTinterwordspacing}{\spaceskip=\fontdimen2\font plus
\BIBentryALTinterwordstretchfactor\fontdimen3\font minus \fontdimen4\font\relax}
\providecommand{\BIBforeignlanguage}[2]{{%
\expandafter\ifx\csname l@#1\endcsname\relax
\typeout{** WARNING: IEEEtran.bst: No hyphenation pattern has been}%
\typeout{** loaded for the language `#1'. Using the pattern for}%
\typeout{** the default language instead.}%
\else
\language=\csname l@#1\endcsname
\fi
#2}}
\providecommand{\BIBdecl}{\relax}
\BIBdecl

\bibitem{Learn1Day}
A.~Kendall, J.~Hawke, D.~Janz, P.~Mazur, D.~Reda, J.~M. Allen, V.~D. Lam, A.~Bewley, and A.~Shah, ``Learning to drive in a day,'' in \emph{2019 International Conference on Robotics and Automation (ICRA)}, 2019, pp. 8248--8254.

\bibitem{E2ESelf}
M.~Bojarski, D.~W. del Testa, D.~Dworakowski, B.~Firner, B.~Flepp, P.~Goyal, L.~D. Jackel, M.~Monfort, U.~Muller, J.~Zhang, X.~Zhang, J.~Zhao, and K.~Zieba, ``End to end learning for self-driving cars,'' \emph{ArXiv}, vol. abs/1604.07316, 2016.

\bibitem{DenseRL}
S.~Feng, H.~Sun, X.~Yan, H.~Zhu, Z.~Zou, S.~Shen, and H.~X. Liu, ``Dense reinforcement learning for safety validation of autonomous vehicles,'' \emph{Nature}, vol. 615, no. 7953, pp. 620--627, 2023.

\bibitem{SurveyRLIL01}
S.~Aradi, ``Survey of deep reinforcement learning for motion planning of autonomous vehicles,'' \emph{IEEE Transactions on Intelligent Transportation Systems}, vol.~23, no.~2, pp. 740--759, 2022.

\bibitem{SurveyRLIL02}
Z.~Zhu and H.~Zhao, ``A survey of deep {RL} and {IL} for autonomous driving policy learning,'' \emph{IEEE Transactions on Intelligent Transportation Systems}, vol.~23, no.~9, pp. 4043--4065, 2022.

\bibitem{survey4}
X.~Di and R.~Shi, ``A survey on autonomous vehicle control in the era of mixed-autonomy: From physics-based to ai-guided driving policy learning,'' \emph{Transportation Research Part C: Emerging Technologies}, vol. 125, pp. 3008--3048, 2021.

\bibitem{survey5}
K.~Muhammad, A.~Ullah, J.~Lloret, J.~D. Ser, and V.~H.~C. de~Albuquerque, ``Deep learning for safe autonomous driving: Current challenges and future directions,'' \emph{IEEE Transactions on Intelligent Transportation Systems}, vol.~22, no.~7, pp. 4316--4336, 2021.

\bibitem{RLSurvey01}
S.~Aradi, ``Survey of deep reinforcement learning for motion planning of autonomous vehicles,'' \emph{IEEE Transactions on Intelligent Transportation Systems}, vol.~23, no.~2, pp. 740--759, 2022.

\bibitem{RLSurvey02}
D.~Silver, T.~Hubert, J.~Schrittwieser, I.~Antonoglou, M.~Lai, A.~Guez, M.~Lanctot, L.~Sifre, D.~Kumaran, T.~Graepel, T.~Lillicrap, K.~Simonyan, and D.~Hassabis, ``A general reinforcement learning algorithm that masters chess, shogi, and go through self-play,'' \emph{Science}, vol. 362, no. 6419, pp. 1140--1144, 2018.

\bibitem{RLSurvey03}
B.~R. Kiran, I.~Sobh, V.~Talpaert, P.~Mannion, A.~A.~A. Sallab, S.~Yogamani, and P.~Pérez, ``Deep reinforcement learning for autonomous driving: A survey,'' \emph{IEEE Transactions on Intelligent Transportation Systems}, vol.~23, no.~6, pp. 4909--4926, 2022.

\bibitem{Roach}
Z.~Zhang, A.~Liniger, D.~Dai, F.~Yu, and L.~Van~Gool, ``End-to-end urban driving by imitating a reinforcement learning coach,'' in \emph{2021 IEEE/CVF International Conference on Computer Vision (ICCV)}, 2021, pp. 5202--5212.

\bibitem{RLReward01}
W.~B. Knox, A.~Allievi, H.~Banzhaf, F.~Schmitt, and P.~Stone, ``Reward misdesign for autonomous driving,'' \emph{Artif. Intell.}, vol. 316, no. 103829, Mar. 2023.

\bibitem{RLReward02}
J.~Leike, D.~Krueger, T.~Everitt, M.~Martic, V.~Maini, and S.~Legg, ``Scalable agent alignment via reward modeling: A research direction. arxiv 2018,'' \emph{arXiv preprint arXiv:1811.07871}, 1811.

\bibitem{RLReward03}
A.~Bondarenko, D.~Volk, D.~Volkov, and J.~Ladish, ``Demonstrating specification gaming in reasoning models,'' \emph{arXiv preprint arXiv:2502.13295}, 2025.

\bibitem{RLSample01}
J.~Wu, Z.~Huang, Z.~Hu, and C.~Lv, ``Toward human-in-the-loop {AI}: Enhancing deep reinforcement learning via real-time human guidance for autonomous driving,'' \emph{Engineering}, vol.~21, pp. 75--91, 2023.

\bibitem{RLSample02}
A.~Harutyunyan, W.~Dabney, T.~Mesnard, M.~Gheshlaghi~Azar, B.~Piot, N.~Heess, H.~P. van Hasselt, G.~Wayne, S.~Singh, D.~Precup \emph{et~al.}, ``Hindsight credit assignment,'' vol.~32, 2019, pp. 167--175.

\bibitem{RLSafe}
W.~Saunders, G.~Sastry, A.~Stuhlmueller, and O.~Evans, ``Trial without error: Towards safe reinforcement learning via human intervention,'' \emph{arXiv preprint arXiv:1707.05173}, 2017.

\bibitem{ILSurvey01}
L.~Le~Mero, D.~Yi, M.~Dianati, and A.~Mouzakitis, ``A survey on imitation learning techniques for end-to-end autonomous vehicles,'' \emph{IEEE Transactions on Intelligent Transportation Systems}, vol.~23, no.~9, pp. 4128--4147, 2022.

\bibitem{ILIRL01}
Z.~Huang, H.~Liu, J.~Wu, and C.~Lv, ``Conditional predictive behavior planning with inverse reinforcement learning for human-like autonomous driving,'' \emph{IEEE Transactions on Intelligent Transportation Systems}, vol.~24, no.~7, pp. 7244--7258, 2023.

\bibitem{IL01}
B.~Widrow, ``Pattern recognition and adaptive control,'' \emph{IEEE Transactions on Applications and Industry}, vol.~83, no.~74, pp. 269--277, 1964.

\bibitem{IL02}
T.~Osa, J.~Pajarinen, G.~Neumann, J.~A. Bagnell, P.~Abbeel, and J.~Peters, ``An algorithmic perspective on imitation learning,'' \emph{ArXiv}, vol. abs/1811.06711, 2018.

\bibitem{IL03}
J.~Huang, S.~Xie, J.~Sun, Q.~Ma, C.~Liu, D.~Lin, and B.~Zhou, ``Learning a decision module by imitating driver’s control behaviors,'' in \emph{Proceedings of the 2020 Conference on Robot Learning}, J.~Kober, F.~Ramos, and C.~Tomlin, Eds., 2021, pp. 1--10.

\bibitem{IL04}
H.~Bharadhwaj, A.~Kumar, N.~Rhinehart, S.~Levine, F.~Shkurti, and A.~Garg, ``Conservative safety critics for exploration,'' \emph{ArXiv}, vol. abs/2010.14497, 2020.

\bibitem{IL05}
Y.~Wu, G.~Tucker, and O.~Nachum, ``Behavior regularized offline reinforcement learning,'' \emph{ArXiv}, vol. abs/1911.11361, 2019.

\bibitem{IL06}
S.~Fujimoto, D.~Meger, and D.~Precup, ``Off-policy deep reinforcement learning without exploration,'' in \emph{Proceedings of the 36th International Conference on Machine Learning}, vol.~97, 2019, pp. 2052--2062.

\bibitem{IRL01}
J.~Sun, L.~Yu, P.~Dong, B.~Lu, and B.~Zhou, ``Adversarial inverse reinforcement learning with self-attention dynamics model,'' \emph{IEEE Robotics and Automation Letters}, vol.~6, no.~2, pp. 1880--1886, 2021.

\bibitem{ILSHOURT02}
S.~Ross and D.~Bagnell, ``Efficient reductions for imitation learning,'' in \emph{Proceedings of the thirteenth international conference on artificial intelligence and statistics}.\hskip 1em plus 0.5em minus 0.4em\relax JMLR Workshop and Conference Proceedings, 2010, pp. 661--668.

\bibitem{ILSHOURT01}
F.~Codevilla, E.~Santana, A.~Lopez, and A.~Gaidon, ``Exploring the limitations of behavior cloning for autonomous driving,'' in \emph{2019 IEEE/CVF International Conference on Computer Vision (ICCV)}, 2019, pp. 9328--9337.

\bibitem{ILSHOURT03}
R.~Camacho and D.~Michie, ``Behavioral cloning a correction,'' \emph{AI Mag.}, vol.~16, pp. 92--101, 1995.

\bibitem{ILSHOURT04}
D.~Silver, T.~Hubert, J.~Schrittwieser, I.~Antonoglou, M.~Lai, A.~Guez, M.~Lanctot, L.~Sifre, D.~Kumaran, T.~Graepel, T.~P. Lillicrap, K.~Simonyan, and D.~Hassabis, ``A general reinforcement learning algorithm that masters chess, shogi, and go through self-play,'' \emph{Science}, vol. 362, pp. 1140--1144, 2018.

\bibitem{DAGGER}
S.~Ross, G.~Gordon, and D.~Bagnell, ``A reduction of imitation learning and structured prediction to no-regret online learning,'' in \emph{International Conference on Artificial Intelligence and Statistics}, 2011, pp. 627--635.

\bibitem{DAGGEROther01}
J.~Zhang and K.~Cho, ``Query-efficient imitation learning for end-to-end autonomous driving,'' \emph{ArXiv}, vol. abs/1605.06450, 2016.

\bibitem{DAGGEROther02}
M.~Kelly, C.~Sidrane, K.~Driggs-Campbell, and M.~J. Kochenderfer, ``Hg-dagger: Interactive imitation learning with human experts,'' in \emph{2019 International Conference on Robotics and Automation (ICRA)}, 2019, pp. 8077--8083.

\bibitem{DAGGEROther03}
R.~Hoque, A.~Balakrishna, E.~Novoseller, A.~Wilcox, D.~S. Brown, and K.~Goldberg, ``Thriftydagger: Budget-aware novelty and risk gating for interactive imitation learning,'' in \emph{Proceedings of the 5th Conference on Robot Learning}, vol. 164, 2022, pp. 598--608.

\bibitem{EIL}
J.~Spencer, S.~Choudhury, M.~Barnes, M.~Schmittle, M.~Chiang, P.~J. Ramadge, and S.~S. Srinivasa, ``Expert intervention learning,'' \emph{Autonomous Robots}, vol.~46, pp. 99--113, 2021.

\bibitem{IWR}
A.~Mandlekar, D.~Xu, R.~Mart'in-Mart'in, Y.~Zhu, F.~F. Li, and S.~Savarese, ``Human-in-the-loop imitation learning using remote teleoperation,'' \emph{ArXiv}, vol. abs/2012.06733, 2020.

\bibitem{HIL01}
P.~F. Christiano, J.~Leike, T.~Brown, M.~Martic, S.~Legg, and D.~Amodei, ``Deep reinforcement learning from human preferences,'' in \emph{Advances in Neural Information Processing Systems}, 2017, pp. 4299--4307.

\bibitem{HIL02}
E.~Biyik and D.~Sadigh, ``Batch active preference-based learning of reward functions,'' in \emph{Proceedings of The 2nd Conference on Robot Learning}, vol.~87, 2018, pp. 519--528.

\bibitem{HIL03}
M.~Palan, N.~C. Landolfi, G.~Shevchuk, and D.~Sadigh, ``Learning reward functions by integrating human demonstrations and preferences,'' \emph{ArXiv}, vol. abs/1906.08928, 2019.

\bibitem{HACO}
Q.~Li, Z.~Peng, and B.~Zhou, ``Efficient learning of safe driving policy via human-{AI} copilot optimization,'' in \emph{International Conference on Learning Representations}, 2022, pp. 1--19.

\bibitem{PVP}
Z.~Peng, W.~Mo, C.~Duan, Q.~Li, and B.~Zhou, ``Learning from active human involvement through proxy value propagation,'' in \emph{Advances in Neural Information Processing Systems}, 2023, pp. 7969--7992.

\bibitem{HILSHORT01}
\BIBentryALTinterwordspacing
T.~Yu, Z.~He, D.~Quillen, R.~Julian, K.~Hausman, C.~Finn, and S.~Levine, ``Meta-world: A benchmark and evaluation for multi-task and meta reinforcement learning,'' 2019. [Online]. Available: \url{https://arxiv.org/abs/1910.10897}
\BIBentrySTDinterwordspacing

\bibitem{HILSHORT03}
R.~Krishna, D.~Lee, L.~Fei-Fei, and M.~S. Bernstein, ``Socially situated artificial intelligence enables learning from human interaction,'' vol. 119, no.~39, 2022, pp. 1157--1169.

\bibitem{SAC}
T.~Haarnoja, A.~Zhou, P.~Abbeel, and S.~Levine, ``Soft actor-critic: Off-policy maximum entropy deep reinforcement learning with a stochastic actor,'' in \emph{International Conference on Machine Learning}, 2018, pp. 1861--1870.

\bibitem{HAIM}
Z.~Huang, Z.~Sheng, C.~Ma, and S.~Chen, ``Human as {AI} mentor: Enhanced human-in-the-loop reinforcement learning for safe and efficient autonomous driving,'' \emph{Communications in Transportation Research}, vol.~4, pp. 100--127, 2024.

\bibitem{lemma2}
Z.~Xue, Z.~Peng, Q.~Li, Z.~Liu, and B.~Zhou, ``Guarded policy optimization with imperfect online demonstrations,'' 2023.

\bibitem{trpo}
J.~Schulman, S.~Levine, P.~Moritz, M.~I. Jordan, and P.~Abbeel, ``Trust region policy optimization,'' 2017.

\bibitem{metadrive}
Q.~Li, Z.~Peng, L.~Feng, Q.~Zhang, Z.~Xue, and B.~Zhou, ``Metadrive: Composing diverse driving scenarios for generalizable reinforcement learning,'' \emph{IEEE Transactions on Pattern Analysis and Machine Intelligence}, vol.~45, no.~3, pp. 3461--3475, 2023.

\bibitem{PPO_LAG}
A.~Stooke, J.~Achiam, and P.~Abbeel, ``Responsive safety in reinforcement learning by pid lagrangian methods,'' in \emph{International Conference on Machine Learning}, 2020, pp. 9133--9143.

\bibitem{SAC_LAG}
S.~Ha, P.~Xu, Z.~Tan, S.~Levine, and J.~Tan, ``Learning to walk in the real world with minimal human effort,'' \emph{arXiv preprint arXiv:2002.08550}, 2020.

\bibitem{CQL}
A.~Kumar, A.~Zhou, G.~Tucker, and S.~Levine, ``Conservative {Q}-learning for offline reinforcement learning,'' in \emph{Proceedings of the 34th International Conference on Neural Information Processing Systems}, 2020, pp. 1179--1191.

\bibitem{PPO}
J.~Schulman, F.~Wolski, P.~Dhariwal, A.~Radford, and O.~Klimov, ``Proximal policy optimization algorithms,'' \emph{arXiv preprint arXiv:1707.06347}, 2017.

\bibitem{DSAC}
J.~Duan, Y.~Guan, S.~E. Li, Y.~Ren, Q.~Sun, and B.~Cheng, ``Distributional soft actor-critic: Off-policy reinforcement learning for addressing value estimation errors,'' \emph{IEEE Transactions on Neural Networks and Learning Systems}, vol.~33, no.~11, pp. 6584--6598, 2022.

\bibitem{BC}
M.~Bain and C.~Sammut, ``A framework for behavioural cloning,'' in \emph{Machine Intelligence}, 1999, pp. 103--129.

\bibitem{gail}
J.~Ho and S.~Ermon, ``Generative adversarial imitation learning,'' in \emph{Advances in Neural Information Processing Systems}, D.~Lee, M.~Sugiyama, U.~Luxburg, I.~Guyon, and R.~Garnett, Eds., vol.~29, 2016, pp. 1--9.

\end{thebibliography}


\begin{IEEEbiography}[{\includegraphics[width=1in,height=1.25in,clip,keepaspectratio]{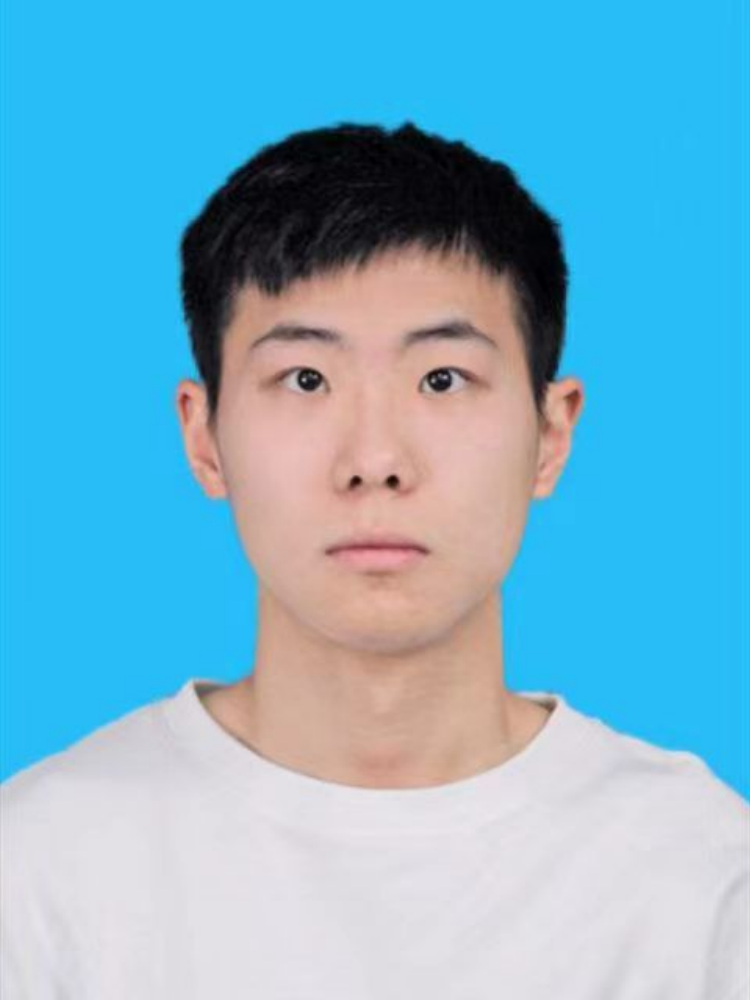}}]{Zeqiao Li}
  received the B.S. degree in intelligent science and technology from Hebei University of Technology, Tianjin, China, in 2023. He is currently pursuing the PhD degree with the School of Electrical and Information Engineering, Tianjin University, Tianjin, China.

  His research interests include reinforcement learning, optimal control, and self-driving decision-making.
  \end{IEEEbiography}

  \vspace{-30pt}
\begin{IEEEbiography}[{\includegraphics[width=1in,height=1.25in,clip,keepaspectratio]{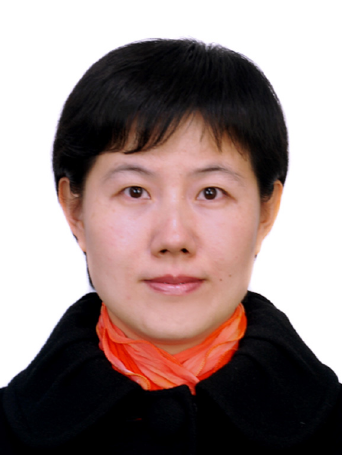}}]{Yijing Wang}  received the M.S. degree in control theory and control engineering from Yanshan University, Qinhuangdao, China, in 2000, and the Ph.D. degree in control theory from Peking University, Beijing, China, in 2004. 
  
  In 2004, she joined the School of Electrical and Information Engineering, Tianjin University, Tianjin, China, where she is a Full Professor. Her research interests are analysis and control of switched/hybrid systems and robust control.
  \end{IEEEbiography}
  \vspace{-30pt}
  

  \begin{IEEEbiography}[{\includegraphics[width=1in,height=1.25in,clip,keepaspectratio]{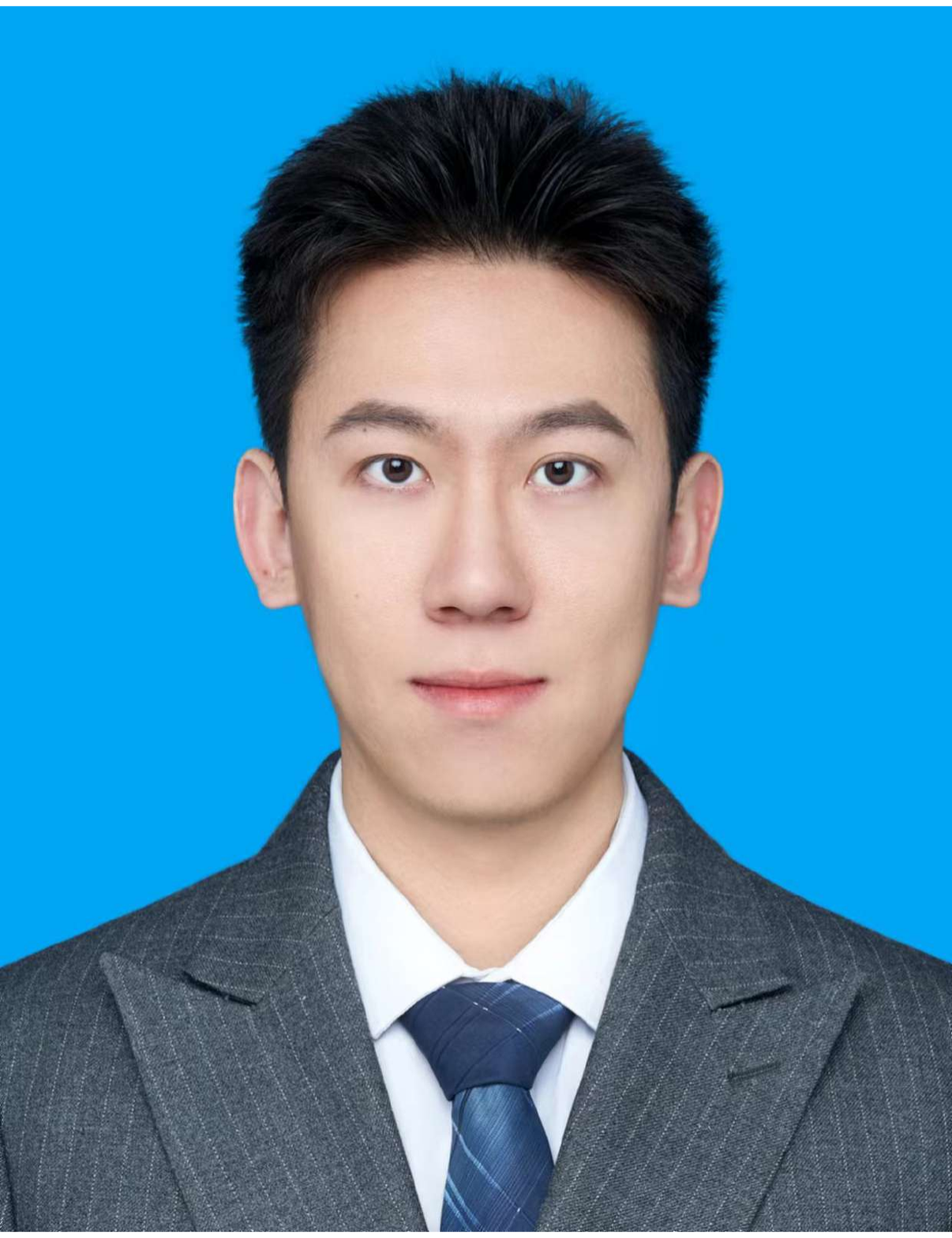}}]{Haoyu Wang}
    received the B.S. degree in automation and the Ph.D. degree in control theory and control engineering from Tianjin University, Tianjin, China, in 2018 and 2023, respectively.
    
    In 2023, he joined the School of Electrical and Information Engineering, Tianjin University, where he is a Postdoctoral Researcher. His research interests include active disturbance rejection control, motion control, and disturbance observer design with application to intelligent vehicles.
    \end{IEEEbiography}
    \vspace{-30pt}
      \begin{IEEEbiography}[{\includegraphics[width=1in,height=1.25in,clip,keepaspectratio]{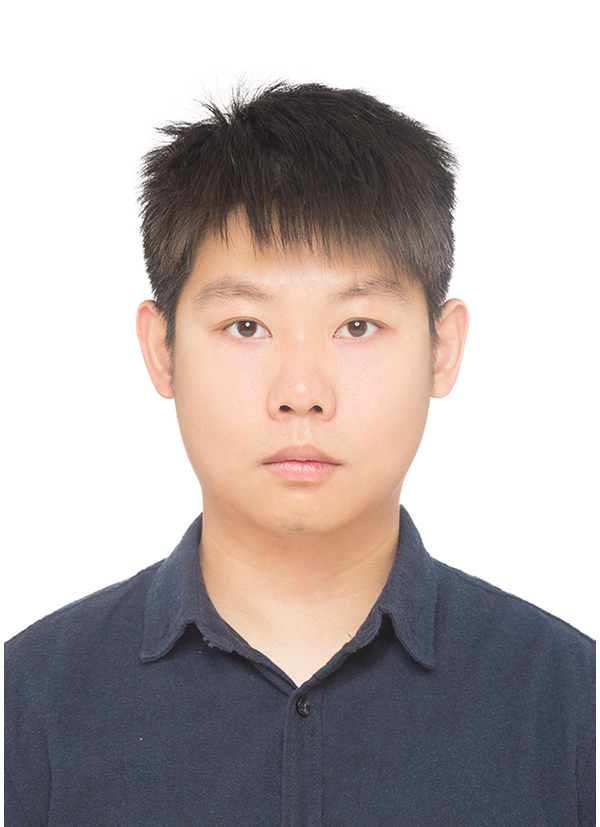}}]{Zheng Li}
received his M.Eng. degree in control engineering from School of Electrical and Information Engineering, Tianjin University, in 2021. He is currently pursuing the Ph.D. degree with the Tianjin Key Laboratory of Intelligent Unmanned Swarm Technology and System, Tianjin University. From January 2024 to December 2024, he worked as a visiting Ph.D. student in the Department of Mechanical Engineering, University of Victoria, Victoria BC, Canada. \\
 His research interests including interactive multi-task prediction, maneuver decision, trajectory planning, automatic control and vehicular simulation and verification for autonomous vehicles.
    \end{IEEEbiography}
    \begin{IEEEbiography}[{\includegraphics[width=1in,height=1.25in,clip,keepaspectratio]{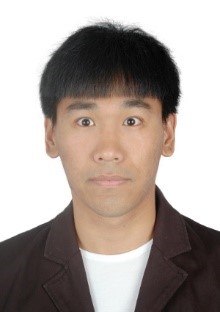}}]{Peng Li} received the Ph.D. degree in control science and engineering in 2024 from Tianjin University, China. He is currently a Research Associate with the School of Electrical and Information Engineering, Tianjin University, China. His research interests include swarm energy systems, networked control systems, wheeled mobile robots, and finite frequency analysis.
    \end{IEEEbiography}
  \begin{IEEEbiography}[{\includegraphics[width=1in,height=1.25in,clip,keepaspectratio]{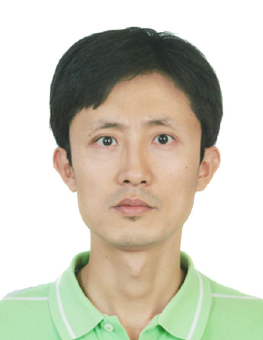}}]{Zhiqiang Zuo}
    (Senior Member, IEEE) received the M.S. degree in control theory and control engineering from Yanshan University, Qinhuangdao, China, in 2001, and the Ph.D. degree in control theory from
    Peking University, Beijing, China, in 2004.
    
    In 2004, he joined the School of Electrical and Information Engineering, Tianjin University, Tianjin, China, where he is a Full Professor. From 2008 to 2010, he was a Research Fellow with the Department of Mathematics, City University of Hong Kong, Hong Kong. From 2013 to 2014, he was a Visiting Scholar with the University of California at Riverside, Riverside, CA, USA. His research interests include nonlinear control, robust control, and multiagent systems, with application to intelligent vehicles.
    \end{IEEEbiography}
      \begin{IEEEbiography}[{\includegraphics[width=1in,height=1.25in,clip,keepaspectratio]{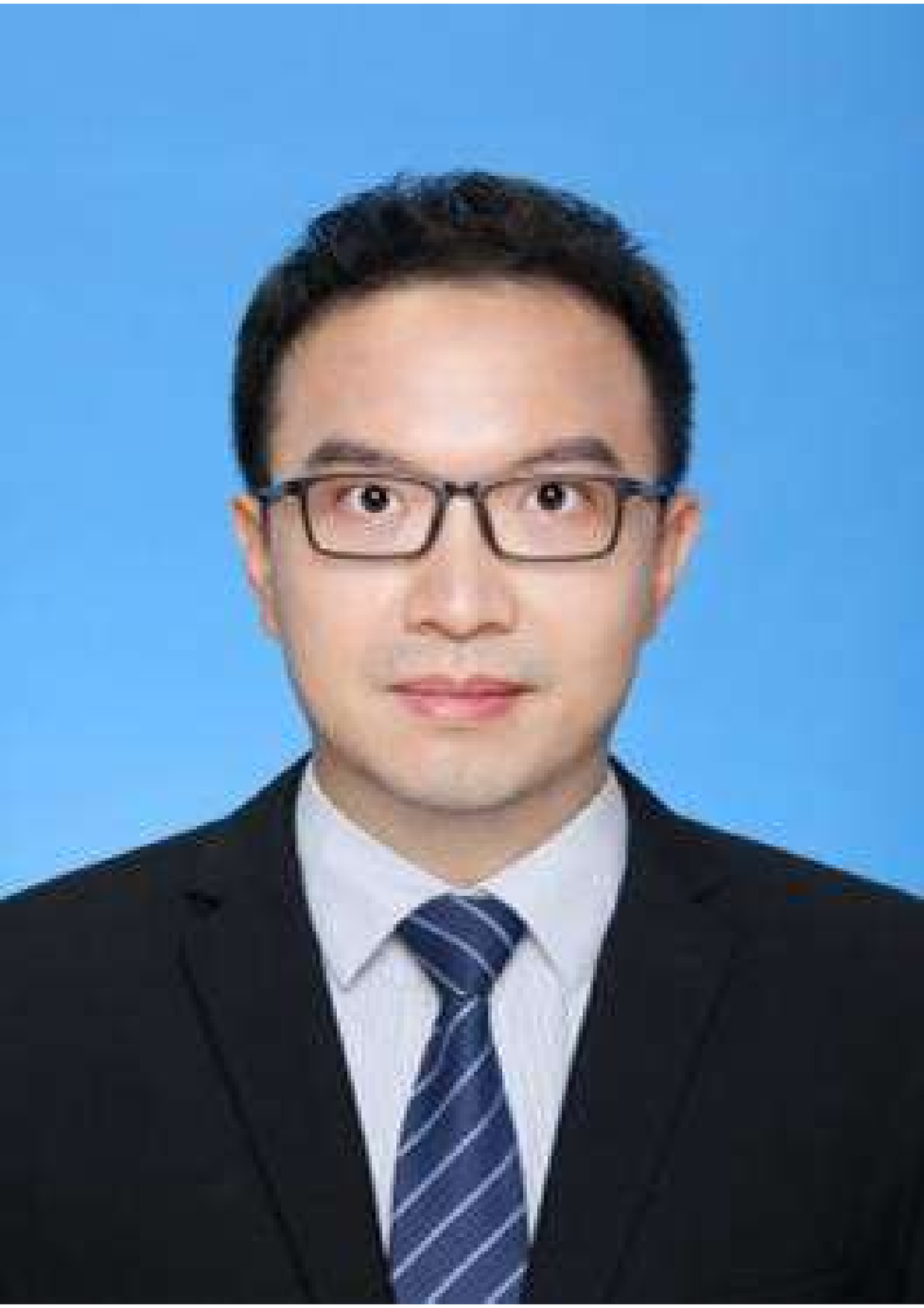}}]{Chuan Hu}
    received the B.S. degree in automotive engineering from Tsinghua University, Beijing, China, in 2010, the M.S. degree in vehicle operation engineering from the China Academy of Railway Sciences, Beijing, in 2013, and the Ph.D. degree in
    mechanical engineering from McMaster University, Hamilton, ON, Canada, in 2017. He has been a tenure-track Associate Professor with the School of Mechanical Engineering, Shanghai Jiao Tong University, Shanghai, China, Since July 2022. His
    research interests include the perception, decisionmaking,path planning, and motion control of intelligent and connected vehicles (ICVs), autonomous driving, eco-driving, human–machine trust an cooperation, shared control, and machine-learning applications in ICVs.
    \end{IEEEbiography}

\vfill

\end{document}